\documentclass{article}
\usepackage[utf8]{inputenc}

\usepackage{graphicx}
\usepackage{amsmath,amssymb,amsthm,bm,bbm,mathtools}
\usepackage{enumerate}

\usepackage[accepted]{icml2023}

\usepackage{hyperref}
\usepackage{xcolor}
\definecolor{darkblue}{rgb}{0, 0, 0.5}
\hypersetup{colorlinks=true, citecolor=darkblue, linkcolor=darkblue, urlcolor=darkblue}

\newcommand{\RR}{\mathbb{R}}

\newcommand{\PP}{\mathbb{P}}
\newcommand{\EE}{\mathbb{E}}
\newcommand{\MM}{\mathbb{M}}
\newcommand{\QQ}{\mathbb{Q}}

\newcommand{\cG}{\mathcal{G}}
\newcommand{\cL}{\mathcal{L}}

\newcommand{\eps}{\varepsilon}

\newcommand{\simiid}{\overset{\textrm{i.i.d.}}{\sim}}

\theoremstyle{plain}
\newtheorem{thm}{Theorem}[section]
\newtheorem{lem}[thm]{Lemma}

\newtheorem{cor}[thm]{Corollary}
\theoremstyle{definition}

\DeclareMathOperator*{\argmin}{argmin}

\icmltitlerunning{One-sided Matrix Completion from Two Observations Per Row}

\begin{document}

\twocolumn[
\icmltitle{One-sided Matrix Completion from Two Observations Per Row}



\icmlsetsymbol{equal}{*}

\begin{icmlauthorlist}
\icmlauthor{Steven Cao}{sch}
\icmlauthor{Percy Liang}{sch}
\icmlauthor{Gregory Valiant}{sch}
\end{icmlauthorlist}

\icmlaffiliation{sch}{Stanford University}

\icmlcorrespondingauthor{Steven Cao}{shcao@stanford.edu}

\icmlkeywords{Machine Learning, ICML, Matrix Completion, Subspace Estimation, High-dimensional Statistics, Random Matrix Theory}

\vskip 0.3in
]

\printAffiliationsAndNotice{}  

\begin{abstract}

Given only a few observed entries from a low-rank matrix $X$, matrix completion is the problem of imputing the missing entries, and it formalizes a wide range of real-world settings that involve estimating missing data. However, when there are too few observed entries to complete the matrix, what other aspects of the underlying matrix can be reliably recovered? We study one such problem setting, that of ``one-sided'' matrix completion, where our goal is to recover the right singular vectors of $X$, even in the regime where recovering the left singular vectors is impossible, which arises when there are more rows than columns and very few observations. We propose a natural algorithm that involves imputing the missing values of the matrix $X^TX$ and show that even with only two observations per row in $X$, we can provably recover $X^TX$ as long as we have at least $\Omega(r^2 d \log d)$ rows, where $r$ is the rank and $d$ is the number of columns. We evaluate our algorithm on one-sided recovery of synthetic data and low-coverage genome sequencing. In these settings, our algorithm substantially outperforms standard matrix completion and a variety of direct factorization methods.

\end{abstract}

\section{Introduction}

Matrix completion, the problem of recovering a low-rank matrix after observing only a subset of its entries, formalizes a wide range of real-world settings that involve estimating missing data, including recommending movies to users~\citep{koren2009matrix}, reducing MRI scan time via parallel imaging~\citep{shin2014calibrationless}, and quantifying annotator disagreement in dataset crowdsourcing~\citep{gordon2021disagreeement}. Over the years, a flurry of research has produced a robust understanding of the problem~\citep[\textit{inter alia}]{candes2009exact,keshavan2009matrix}. However, most of our understanding is restricted to settings where each row and each column have more observations than the rank of the underlying matrix. It is natural that past work operated under this assumption because full matrix completion is impossible without it: for a rank-$r$ matrix $X$ with shape $m \times d$, one can show that estimating the matrix is impossible with $o(r(m + d))$ observations. Nonetheless, many important applications do not satisfy this assumption: for example, in low-coverage genotype imputation~\citep{li2009genotype}, we might sequence $d=2{,}000$ people for $10{,}000$ genetic variants each, out of the $m = 10{,}000{,}000$ genetic variants in humans. Represented as a matrix, we have a $10{,}000{,}000 \times 2{,}000$ matrix with $2{,}000 * 10{,}000 = 20{,}000{,}000$ total observations, or about two observations per row on average, which is certainly much less than the rank of the matrix. Given too few observed entries to fully complete the matrix, what other aspects of the underlying matrix can be reliably recovered?

\begin{figure*}[t]
\begin{center}
    \includegraphics[width=\linewidth]{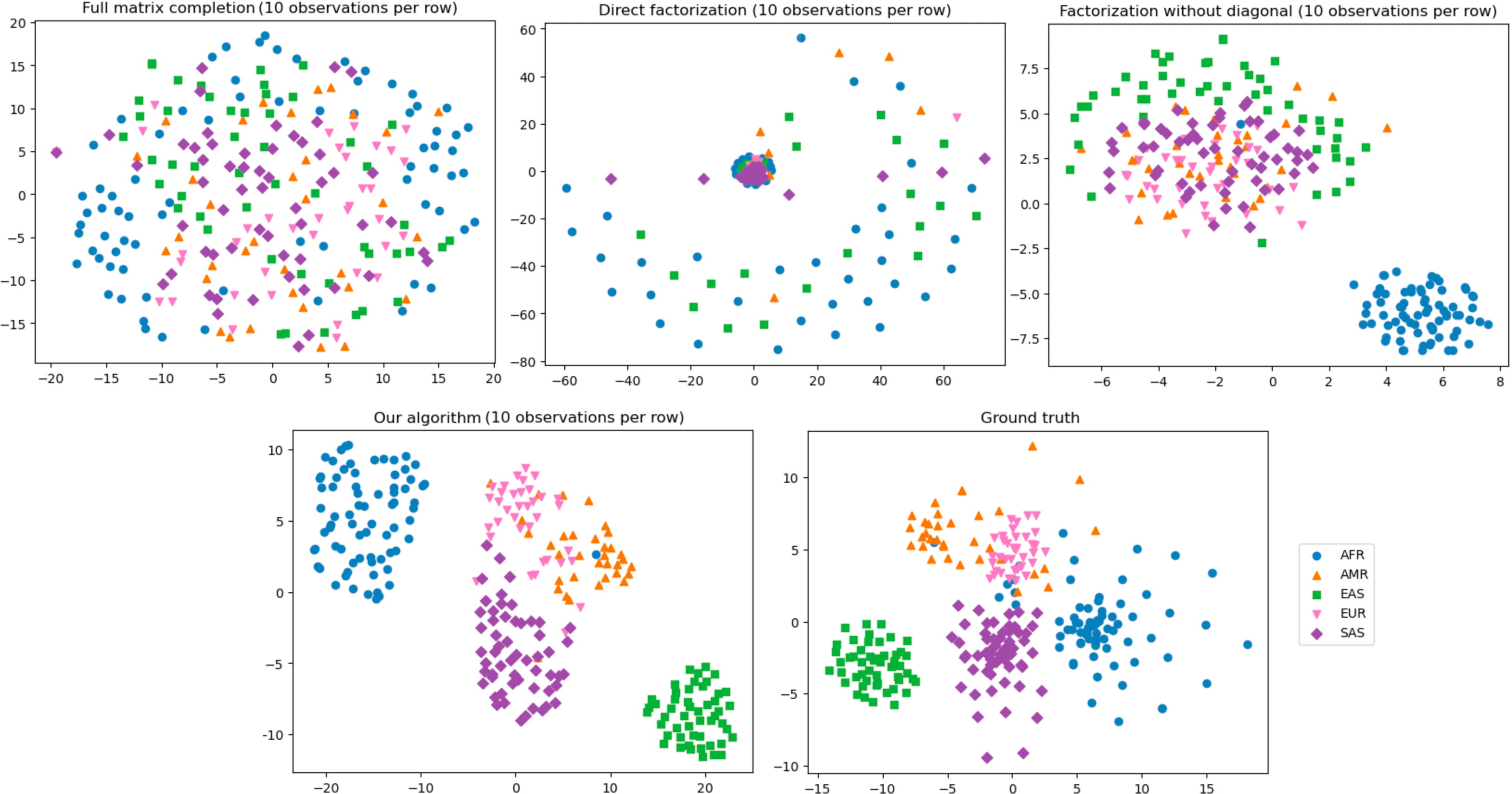}
\end{center}
\caption{\label{figure:tsne-1} TSNE~\citep{vandermaaten2008visualizing} visualization of column factor (i.e.\ eigenvalue-weighted right singular vector) recovery on the 1000genomes dataset~\citep{fairly2019international} with $m = 1\,500\,000$ bialleles, $d = 250$ people, and $k = 10$ observations per row. Each of the $d$ column factors can be thought of as a vector representation of each person and their underlying genotype variation factors, with ethnicity being the strongest contributor. Colors represent ethnicity (AFR: blue, AMR: orange, EAS: green, EUR: pink, SAS: purple). Methods (left to right, top to bottom): (1) Full matrix completion: factor the output of matrix completion on $P_E(X)$ (where given the observation mask $E$, $P_E$ sets unobserved entries to zero), (2) Direct factorization: factor $P_E(X)^TP_E(X)$, (3) Factorization without diagonal: factor $P_E(X)^TP_E(X)$ with the diagonal set to zero~\citep{cai2021subspace}, (4) Our algorithm: perform matrix completion to estimate $X^TX$ (see Equation~\ref{eqn:program-main}) and then factor it, and (5) Ground truth: factor the original fully observed matrix. Our method produces column factors closest to the ground truth, while full matrix completion fails because there are too few observations per row.}
\end{figure*}

In this paper, we study settings like these and show that even with just \emph{two} entries per row, as long as there are sufficiently many rows, we can perform ``one-sided'' recovery and estimate the right singular vectors $Q \in \RR^{d \times r}$. In the aforementioned genomics example, this means recovering the $r$ underlying genotype variation factors for each of the $d$ people (e.g.\ ethnicity, sex, and so on; see Figure~\ref{figure:tsne-1}). This result is despite the fact that we have close to no information about the left singular vectors $P \in \RR^{m \times r}$.

This result might seem counterintuitive due to how ill-posed the recovery problem is. In particular, let $X = UV^T$ denote the ground truth rank-$r$ matrix with factors $U \in \RR^{m \times r}$ and $V \in \RR^{d \times r}$. Then, given fewer than $r$ observations per row, for \emph{any} $\tilde V \in \RR^{d \times r}$ for which every subset of $r$ rows is linearly independent, we can find a $\tilde U \in \RR^{m \times r}$ using linear inversion such that $\tilde X=\tilde U\tilde V^T$ agrees with the observations.\footnote{For example, suppose we are given exactly two observations per row $X_{i,a(i)}$ and $X_{i,b(i)}$, where $i$ is the row index and $a(i), b(i)$ are the two observed locations. Then, given any pairwise linearly independent set of rank-2 vectors $v_1, ..., v_d \in \RR^{2}$, we can choose each $u_1, ..., u_m \in \RR^2$ by inverting the constraints $\langle u_i, v_{a(i)} \rangle = X_{i, a(i)}$ and $\langle u_i, v_{b(i)} \rangle = X_{i, b(i)}$. Then, stacking the $u_i$'s and $v_i$'s, we have a matrix $\tilde X = \tilde U \tilde V^T$ that agrees with the data.} In other words, regardless of how many rows we have, observing two entries per row is not enough to distinguish between a wide range of low-rank matrices, each with drastically different right singular vectors. 

So then how is recovery possible? Recall that when constructing the distractors $\tilde X$, the left factors $\tilde U$ were chosen via linear inversion, making the observation locations and $\bar X$ dependent on one another. Then, to rule out these distractors, we can make the assumption that the observation locations are chosen randomly, independently of the underlying matrix. Intuitively, this assumption lets us rule out matrices that are too ``correlated'' with the observation locations: for example, we can rule out candidate matrices whose entries consistently are smaller at the observation locations than outside, which is often the case for distractor matrices constructed by linear inversion. While our algorithm does not explicitly rule out matrices in this way, we provide this discussion to emphasize the importance of random sampling in this setting and provide intuition for why the apparent ill-posedness is not fatal.

Algorithmically, our main idea is that each pair of observations in a row, denoted $X_{i, a(i)}$ and $X_{i, b(i)}$, produces a noisy estimate of entry $(a(i), b(i))$ of the matrix $\frac{1}{m} X^TX$. Then, given enough rows, we have sufficiently many observations to impute the missing entries of $\frac{1}{m} X^TX$, which we can then factor to obtain the right-side singular vectors (or equivalently, recover the rowspace) of $X$. Specifically, we find that $m = \Omega(\alpha^2 rd \log d)$ rows is sufficient to complete $\frac{1}{m} X^TX$ in additive Frobenius error, where $\alpha$ denotes the maximum squared entry $\alpha = \max_{ij} X_{ij}^2$. From this bound, we show that rowspace recovery is possible with $m = \Omega(r^2 d \log d)$ rows, with synthetic experiments suggesting that the $r^2$ dependence is fundamental.

The idea of operating on $X^TX$ when $X$ is unbalanced has been explored in a variety of papers on noisy matrix factorization and subspace estimation, problems closely related to ours~\citep[\textit{inter alia}]{gonen2016subspace,donoho2022optimal,montanari2022fundamental}. Broadly speaking, these papers analyze the error of directly factoring a noisy or incomplete version of $X^TX$, while we focus on showing how the missing values of $X^TX$ can be imputed. The papers most directly related to ours are \citet{montanari-sun-2018} and \citet{cai2021subspace}, who study direct factorization for subspace estimation from incomplete observations. The main difference with our work is that they assume that the observations are uniform over the matrix, making it unclear how reliant the algorithm is on the heavy rows which happen to have many observations. In contrast, we study one-sided recovery in the setting where \textit{every} row has just two observations. Nonetheless, their direct factorization algorithms are still applicable to our setting. In our experiments, we find that compared to direct factorization, imputing the missing values of $X^TX$ before factoring produces substantially better subspace estimates.

Empirically, we find that our algorithm indeed recovers the right-side singular vectors even when full matrix completion is impossible. We evaluate on synthetic data and the 1000genomes dataset~\citep{fairly2019international}, a real-world example where the matrix is highly incomplete with many more rows than columns. In these settings, our algorithm substantially outperforms standard matrix completion and a variety of direct factorization methods.

\section{Main Result}\label{sec:mainresults}

\textbf{Notation}: for a matrices $A$ and $B$, we use $\|A\|_\text{max}$ to denote $\max_{i,j} |A_{ij}|$, $\|A\|_\text{op}$ to denote the operator norm, $\|A\|_\text{nuc}$ to denote the nuclear norm, $\|A\|_F$ to denote the Frobenius norm, and $\langle A, B \rangle = \text{tr}(A^TB)$ to denote the matrix inner product. We use $E_{i,j} \in \RR^{d \times d}$ to denote the matrix with 1 in entry $(i, j)$ and 0 elsewhere, $f \lesssim g$ to denote that there exists a universal constant $c$ such that $f \leq cg$, and $[d]$ to denote the set $\{1, 2, ..., d \}$. Given an observation mask $E \in \{0,1\}^{m \times d}$ and a matrix $A \in \RR^{m \times d}$, $P_E(A)$ will denote the elementwise multiplication of $E$ and $A$.

The problem setup is as follows: from a rank-$r$ matrix $X \in \RR^{m \times d}$, we randomly observe two entries per row, which we can represent as indices $(a(1), b(1)), ... , (a(m), b(m))$ drawn i.i.d.\ uniformly from $\{(j_1,j_2) : j_1, j_2 \in [d],\ j_1 \neq j_2\}$. We wish to estimate the matrix $\Theta^* = \frac{1}{m}X^TX$.

As an example to provide intuition, let entry $X_{i,j}$ represent the $i$th user's rating for the $j$th item. Writing $X$ as $UV^T$ for $U \in \RR^{m \times r}$ and $V \in \RR^{d \times r}$, each entry $X_{i,j}$ can be written $\langle u_i, v_j \rangle$ where $u_i, v_j \in \RR^r$ are the $i$th and $j$th rows of $U$ and $V$. Then, we can think of $u_i$ as representing the $i$th user's preferences along $r$ attributes, $v_i$ as representing the $j$th item's $r$ attributes, and the rating $X_{ij}$ as their inner product. We'll informally refer to $U$ as row (i.e.\ user) factors and $V$ as column (i.e.\ item) factors (note that this decomposition of $X$ into $U$ and $V$ is not unique). 

Then, in this example, we can write 
\begin{align*}
    \Theta^* &= \frac{1}{m}X^TX = V\bar SV^T,
\end{align*}
where $\bar S$ is given by $\bar S = \frac{1}{m} \sum_{i = 1}^m u_i u_i^T \in \RR^{r \times r}$. Written this way, the $(j_1, j_2)$the entry of $\Theta^*$ is given by $v_{j_1}^T\bar Sv_{j_2}$, which is the inner product (induced by $\bar S$) between the $j_1$th and $j_2$th column factors. Therefore, we can interpret our goal of recovering $\Theta^*$ as recovering a matrix of pairwise ``column factor similarities.''

Why might recovering pairwise similarities be possible? For each pair of observations $X_{i, a(i)} = \langle u_i, v_{a(i)}\rangle$ and $X_{i, b(i)} = \langle u_i, v_{b(i)} \rangle$, while we don't know $u_i$, the two observations should on average be similar if the inner product $\langle v_{a(i)}, v_{b(i)} \rangle_{\bar S}$ is positive, and dissimilar if the inner product is negative. Then, using the product $X_{i, a(i)} X_{i, b(i)}$ as our empirical observation for the similarity between $a(i)$ and $b(i)$, our estimator involves optimizing the following squared loss with a nuclear norm regularizer:\footnote{Note that this program is convex and can be solved via semi-definite programming as is standard in matrix completion. In our experiments, we instead do non-convex gradient descent for computational efficiency, as discussed in Section~\ref{sec:experiments}.}
\begin{align}
    \hat \Theta &\in \argmin_{\| \Theta \|_\text{max} \leq \alpha} \cL(\Theta) + \lambda \| \Theta \|_\text{nuc},\label{eqn:program-main} \\
    \cL(\Theta) &= \frac{1}{4m} \sum_{i=1}^m \Big[ (\Theta_{a(i), b(i)} - X_{i, a(i)} X_{i, b(i)})^2\nonumber \\
    &+ (\Theta_{b(i), a(i)} - X_{i, b(i)} X_{i, a(i)})^2\nonumber \\
    &+ (\Theta_{a(i), a(i)} - X_{i, a(i)}^2)^2\nonumber \\
    &+ (\Theta_{b(i), b(i)} - X_{i, b(i)}^2)^2 \Big].\label{eqn:loss-main}
\end{align}
With this estimator, we have the following error bound:
\begin{thm}[Main result]\label{thm:main-maintext}
Let $\hat \Theta$ be the solution of the optimization problem defined in Equation~\ref{eqn:program-main}, where $\lambda$ is set to $16 \alpha \sqrt{\frac{ \log d + \delta}{dm}}$. Also, suppose that $X$ is rank $r$ with $\| X \|_\text{max}^2 \leq \alpha$, and $m \geq d(\log d + \delta)$. Then, with probability $\geq 1 - 3e^{-\delta}$, we have that
\begin{align*}
    \frac{1}{d^2} \| \hat \Theta - \Theta^* \|_F^2 \lesssim \alpha^2 \frac{rd(\log d + \delta)}{m}.
\end{align*}
\end{thm}
From this theorem, we can derive the following two corollaries: first, because $X$ and $\Theta^* = \frac{1}{m} X^TX$ have the same rowspace, we can use recovery of $\Theta^*$ to estimate the rowspace of $X$. As is standard, we can measure rowspace recovery error as the error in estimating the right-side singular vectors up to rotation, producing the following:
\begin{cor}[Right-side singular vector recovery]\label{cor:rowspace-maintext}
    Under the same conditions as Theorem~\ref{thm:main-maintext}, let $Q \in \RR^{d \times r}$ denote the right-side singular vectors of $X$, and let $\hat Q \in \RR^{d \times r}$ be the top $r$ singular vectors of $\hat \Theta$. Then, letting $\sigma_r$ be the $r$th singular value of $\Theta^* = \frac{1}{m}X^TX$, we have
    \begin{align*}
        \min_{R \in \RR^{r \times r}: R^TR = I_r} \| \hat Q R - Q \|_F^2 \lesssim \left(\frac{d \alpha}{\sigma_r}\right)^2 \frac{rd(\log d + \delta)}{m}.
    \end{align*}
\end{cor}
In this corollary, a $\sigma_r$ factor appears in the denominator because the algorithm is tasked with recovering all $r$ singular directions, even if some have low weight (i.e.\ $\sigma_r$ is small). However, in many applications, we often care only about recovering the singular directions with high weight, as the low weight singular directions have little effect on the data. Therefore, from Theorem~\ref{thm:main-maintext} we can also derive a column (i.e.\ right-side) factor recovery result, which can be thought of as a weighted version of the rowspace recovery result:
\begin{cor}[Column factor recovery]\label{cor:colfactor-maintext}
    Under the same conditions as Theorem~\ref{thm:main-maintext}, let $\hat Q \in \RR^{d \times r}$ and $\hat \Lambda \in \RR^{r \times r}$ be the top $r$ singular vectors and singular values of $\hat \Theta$. Also, let $\Theta^* = Q \Lambda Q^T$ be the SVD of $\Theta^* = \frac{1}{m}X^TX$, where $Q \in \RR^{d \times r}$ and $\Lambda \in \RR^{r \times r}$. Then, we have
    \begin{align*}
        \min_{\substack{R \in \RR^{r\times r}:\\ R^TR = I_r}} \frac{1}{d} \| \hat Q \hat \Lambda^{1/2} R - Q \Lambda^{1/2} \|_F^2 &\lesssim \alpha \sqrt{\frac{r^2 d(\log d + \delta)}{m}}.
    \end{align*}
\end{cor}
We refer to this corollary as ``column factor recovery'' because in the users and items example, we can think of the corollary as using our estimate of $\Theta^* = V \bar S V^T$ to recover $V \bar S^{1/2} \in \RR^{d \times r}$, or the column factors ``skewed'' by $\bar S$. 

\subsection{Interpreting the bounds}\label{subsec:interpreting}

While the error metrics in Theorem~\ref{thm:main-maintext} and Corollary~\ref{cor:colfactor-maintext} scale with the magnitude of $X$, the rowspace recovery error in Corollary~\ref{cor:rowspace-maintext} is scale-invariant. Therefore, while Theorem~\ref{thm:main-maintext} and Corollary~\ref{cor:colfactor-maintext} scale with the maximum entry $\alpha$, the bound in Corollary~\ref{cor:rowspace-maintext} is in terms of the scale-invariant quantity $\frac{d \alpha}{\sigma_r}$. To make the bounds more comparable, we can convert these additive bounds into multiplicative ones by dividing both sides by the norm of the quantity being estimated ($\| \Theta^* \|_F^2$ in Theorem~\ref{thm:main-maintext}, $\| Q \|_F^2 = r$ in Corollary~\ref{cor:rowspace-maintext}, and $\| Q \Lambda^{1/2} \|_F^2 = \|\Theta^* \|_\text{nuc}$ in Corollary~\ref{cor:colfactor-maintext}). Then, we have the following sample complexities:
\begin{align*}
    \text{Theorem~\ref{thm:main-maintext}: } m &\gtrsim \left(\frac{d \alpha}{\| \Theta^* \|_F}\right)^2 r d \log d \\
    \text{Corollary~\ref{cor:rowspace-maintext}: } m &\gtrsim \left(\frac{d \alpha}{\sigma_r(\Theta^*) \sqrt{r} }\right)^2 r d \log d \\
    \text{Corollary~\ref{cor:colfactor-maintext}: } m &\gtrsim \left(\frac{d \alpha \sqrt{r}}{\| \Theta^* \|_\text{nuc}}\right)^2 r d \log d.
\end{align*}
Each of these scale-invariant constants can be thought of as incoherence constants which capture the ``spikiness'' of the matrix $\Theta^*$, with similar constants appearing in \citet{negahban2012restricted}, \citet{montanari-sun-2018}, and \citet{cai2021subspace} (among others). Letting $\mu_1 = \frac{d \alpha}{\| \Theta^* \|_F}$, $\mu_2 = \frac{d \alpha}{\sigma_r(\Theta^*) \sqrt{r} }$, and $\mu_3 = \frac{d \alpha \sqrt{r}}{\| \Theta^* \|_\text{nuc}}$, we can interpret these constants by looking at a few examples:
\begin{itemize}
    \item \textbf{All ones matrix}: if $X$ is the all ones matrix, then we have $\mu_1 = \mu_2 = \mu_3 = 1$.
    \item \textbf{Single zero matrix}: if $X$ is $1$ everywhere, except with a $0$ in entry $(1,1)$, then we have $\mu_1$ and $\mu_3$ constant, while $\mu_2$ is order $md$. In other words, for approximate recovery of $\Theta^*$ and the column factors, predicting all ones is sufficient. On the other hand, to recover the rowspace, the algorithm must know where the $0$ is, which requires sampling entry $(1,1)$ and is impossible to do with high probability.
    \item \textbf{Gaussian factors} ($X = UV^T$ with $U, V \simiid \mathcal{N}(0,1)$, $U \in \RR^{m \times r}$, $V \in \RR^{d \times r}$): In this example, $\Theta^* = V \bar S V^T \approx V V^T$ because the empirical covariance $\bar S = \frac{1}{m} \sum_{i=1}^m u_i u_i^T$ is close to the identity $I_r$. Then, up to log and constant factors, we have\footnote{For $r \ll d$, we have (up to constant and log factors) that $\alpha \approx r$, $\| \Theta^* \|_F \approx d \sqrt{r}$, $\sigma_r(\Theta^*) \approx d$, and $\| \Theta^* \|_\text{nuc} \approx r d$.} $\mu_1 \approx \mu_2 \approx \mu_3 \approx \sqrt{r}$, producing a sample complexity of $m \gtrsim r^2 d \log d$. We observe this $r^2$ dependence in synthetic experiments and suspect that it is fundamental.
    \item \textbf{Correlated Gaussian factors}: while the rows of $U$ and $V$ in the previous example were uncorrelated Gaussians, we can also consider the case where they are drawn according to $\mathcal{N}(0, C)$ for some covariance matrix $C \in \RR^{r \times r}$. In this case, $\mu_1$ scales roughly as $\text{tr}(C^2) / \sqrt{\text{tr}(C^4)}$, which is $\sqrt{r}$ if $C$ is the identity but potentially smaller if the eigenvalues of $C$ are non-uniform.\footnote{Let $Z_1 \in \RR^{m \times r}$ and $Z_2 \in \RR^{d \times r}$ have entries drawn i.i.d.\ $\mathcal{N}(0,1)$. Then, for $m \gg r$, we have that $X = Z_1 C Z_2^T$ and $\Theta^* \approx Z_2 C^2 Z_2^T$. The entries of $X$ have mean magnitude $\sqrt{\text{tr}(C^2)}$, and the off-diagonal entries of $\Theta^*$ have mean magnitude $\sqrt{\text{tr}(C^4)}$, with both having sub-Exponential concentration. Then, up to log factors, $\mu_1 = d \alpha / \| \Theta^* \|_F \approx \text{tr}(C^2) / \sqrt{\text{tr}(C^4)}$.} For example, letting $s = (s_1, ..., s_r)$ denote the eigenvalues of $C^2$, if the $s_i$ decay according to a power law $s_i = c_0 i^\alpha$, then $\mu_1$ scales as $\log r$ for $\alpha = 1$ and is constant for $\alpha > 1$. In other words, the sample complexity to recover $\Theta^*$ is $m \gtrsim r d \log d$ if the $r$ factors are sufficiently correlated.
\end{itemize}

\section{Warmup: Gaussian row factors}\label{sec:warmup}

In this section, we warm up with a simpler setting to provide intuition. In particular, suppose that the row factors $u_1, ..., u_m \in \RR^r$ are drawn i.i.d.\ from a standard Gaussian $\mathcal{N}(0, I_r)$. Then, for a pair of observations $X_{i,a(i)}$ and $X_{i, b(i)}$, we have the following expectations:
\begin{align*}
    \EE [X_{i,a(i)} X_{i,b(i)}] &= \EE [u_i^T v_{a(i)} u_i^T v_{b(i)}] \\
    &= v_{a(i)}^T \EE [u_iu_i^T] v_{b(i)} \\
    &= v_{a(i)}^Tv_{b(i)}, \\
    \EE [X_{i,a(i)}^2] &= \| v_{a(i)} \|_2^2.
\end{align*}
Then, each pair of observations gives us unbiased estimates of entries $(a(i), b(i))$, $(b(i), a(i))$, $(a(i), a(i))$, and $(b(i), b(i))$ in the pairwise similarity matrix $VV^T \in \RR^{d \times d}$. Given a very large number of rows, we can then estimate each entry of $VV^T$ as its empirical average:
\begin{align*}
    \text{Off-diagonal } &\text{terms ($j_1 \neq j_2$):} \\
    \hat \Theta_{j_1, j_2}^\text{(emp)} &= \frac{1}{|S_{j_1, j_2}|} \sum_{i \in S_{j_1, j_2}} X_{i, j_1} X_{i, j_2}, \\
    S_{j_1, j_2} &= \{i: (a(i), b(i)) = (j_1, j_2) \text{ or } (j_2, j_1)\}. \\
    \text{Diagonal } &\text{terms ($j_1 = j_2$):} \\
    \hat \Theta_{j, j}^\text{(emp)} &= \frac{1}{|S_{j}|} \sum_{i \in S_{j}} X_{i, j}^2, \\
    S_{j} &= \{i: a(i) = j \text{ or } b(i) = j\}.
\end{align*}
This empirical average can also be written
\begin{align*}
    \hat \Theta^\text{(emp)} &= [P_E(X)^TP_E(X)] \oslash [E^TE],
\end{align*}
where $E \in \RR^{m\times d}$ is the observation mask (1 if that entry is observed, 0 otherwise), $P_E$ sets unobserved entries to zero, and $\oslash$ represents element-wise division for entries where the divisor is non-zero, and a no-op otherwise. In other words, the empirical average $\hat \Theta$ can be written as a renormalized version of $P_E(X)^TP_E(X)$, where the $(j_1, j_2)$th entry of the renormalization matrix $E^TE$ is the number of rows where $j_1$ and $j_2$ are both observed.

If we don't have a very large number of rows, then $\hat \Theta^\text{(emp)}$ can be thought of as a noisy, sparsely populated version of the column factor similarity matrix $VV^T$, which is rank $r$. Then a natural algorithm to estimate $VV^T$ is to minimize the squared loss between the $\hat \Theta$ and $\hat \Theta^\text{(emp)}$, plus a nuclear norm regularizer $\| \hat \Theta \|_\text{nuc}$. Furthermore, we might weight each entry in the squared loss by the number times it was observed, producing the objective
\begin{align*}
    \min_\Theta \| \Theta - \Theta^\text{(emp)} \|_{E^TE}^2 + \lambda \| \Theta \|_\text{nuc},
\end{align*}
where the weighted loss $\| \Theta - \Theta^\text{(emp)} \|_{E^TE}^2$ is given by
\begin{align}\label{eqn:loss-reinterpretation}
    &\| \Theta - \Theta^\text{(emp)} \|_{E^TE}^2 \nonumber \\
    &\qquad = \sum_{j_1 = 1}^d \sum_{j_2 = 1}^d (E^TE)_{j_1, j_2} (\Theta_{j_1, j_2} - \Theta^\text{(emp)}_{j_1, j_2})^2 \nonumber \\
    &\qquad = \text{tr}(\Theta^TE^TE\Theta^\text{(emp)}).
\end{align}
In fact, some calculation shows that this weighted loss is exactly the loss we defined originally in Equation~\ref{eqn:loss-main}, up to rescaling and removal of terms that don't depend on $\Theta$. Therefore, we can interpret our algorithm as performing weighted matrix completion with respect to $\hat \Theta^\text{(emp)}$, which is a properly renormalized version of $P_E(X)^TP_E(X)$.

From this warmup, one can imagine ways to prove noisy matrix completion error bounds for recovery of $VV^T$ given i.i.d.\ Gaussian $u_i$'s, as well as extensions to more general distributions (e.g.\ sub-Gaussian or sub-Exponential), where we instead recover $V \text{Cov}(u) V^T$ where $\text{Cov}(u)$ is the covariance matrix of $u$. However, it is not always reasonable to assume independently drawn row factors: for example, in the genomics case, the chromosome base pairs are certainly not random or independent of one another. Therefore, we would like to prove recovery results without making distributional assumptions on the row factors. 

However, recall from the introduction that the problem is severely underconstrained if we allow arbitrary row factors and masking. Somewhat surprisingly, we find that even with arbitrary row factors, assuming random masking is enough to enable recovery: while the intuition about noisy empirical averages no longer holds, the randomness in the masking is enough to enable a key technical step in the proof involving Radamacher symmetrization (Section~\ref{subsec:opnorm}). In the following section, we provide a sketch of the proof in this more general setting.

\section{Proof sketch}\label{sec:proofsketch}

\begin{figure*}[t]
\begin{center}
    \includegraphics[width=\linewidth]{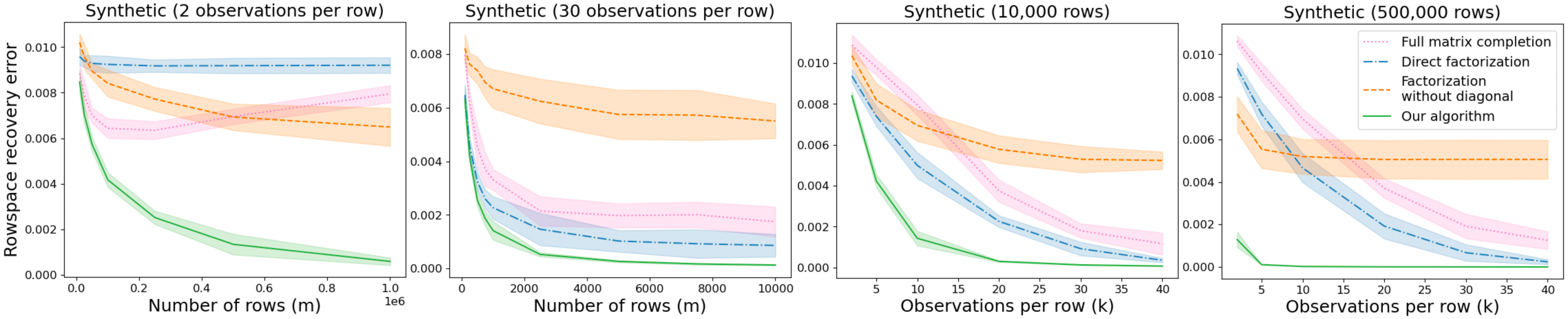}
\end{center}
\caption{\label{figure:synthetic-plot} Rowspace recovery for rank-25 i.i.d.\ Gaussian column and row factors, with $d = 100$ columns, where from left to right, the experiments are as follows: (a) we sample $k=2$ observations per row and vary the number of rows $m$ from $10{,}000$ to $1{,}000{,}000$, (b) $k=30$ with $m$ between $100$ and $10{,}000$ (where fewer rows are necessary because $k$ is large), (c) $m = 10{,}000$ with $k$ between $2$ and $40$, and (d) $m = 500{,}000$ with $k$ between $2$ and $40$. Our algorithm (in solid green) performs produces the most accurate rowspace estimates for all parameter ranges evaluated on, with the gap being largest for large $m$ and small $k$. }
\end{figure*}

\subsection{Outline}
In this section, we outline of the proof, which uses restricted strong convexity arguments~\citep{negahban2012restricted,negahban2012unified}. The proof proceeds as follows: letting $\Delta$ denote the error $\hat \Theta - \Theta^*$, by the optimality of $\hat \Theta$ (along with reverse triangle) we have
\begin{align*}
    0 &\geq \cL(\Theta^* + \Delta) - \cL(\Theta^*) + \lambda ( \| \Theta^* + \Delta \|_\text{nuc} - \| \Theta^* \|_\text{nuc} ) \\
    &\geq \cL(\Theta^* + \Delta) - \cL(\Theta^*) - \lambda \| \Delta \|_\text{nuc}.
\end{align*}
If $\cL$ were $\tau$-strongly-convex, then we could write 
\begin{align*}
    &\geq \langle \nabla \cL(\Theta^*), \Delta \rangle + \tau \| \Delta \|_F^2  - \lambda \| \Delta \|_\text{nuc} \\
    &\geq - \| \nabla \cL(\Theta^*) \|_\text{op} \| \Delta \|_\text{nuc} + \tau \| \Delta \|_F^2  - \lambda \| \Delta \|_\text{nuc},
\end{align*}
where the second inequality follows from Holder's inequality. Then, bounding $\| \nabla \cL(\Theta^*) \|_\text{op}$ and $\| \Delta \|_\text{nuc}$ and rearranging would result in a bound for the error $\| \Delta \|_F^2$. However, $\cL$ is not strongly convex: if we do not restrict to low-rank matrices, then there are multiple matrices that agree with the incomplete observations. Therefore, we'll instead show a form of restricted strong convexity: in particular, we'll show (in Lemma~\ref{lem:rsc}) that the quantity $\cL(\Theta^* + \Delta) - \cL(\Theta^*) - \langle \nabla \cL(\Theta^*), \Delta \rangle$ concentrates around $\frac{1}{d^2} \| \Delta \|_F^2$, but with deviation terms depending on $\| \Delta \|_\text{nuc}$ and $\| \Delta \|_\text{max}$. Therefore, recovery will depend on $\| \Delta \|_\text{nuc}$ being small, which follows from the condition that the regularization strength $\lambda$ is large enough (Lemma~{1} of \citet{negahban2012restricted}), and $\| \Delta \|_\text{max}$ being small, which follows by our assumption that $\| X \|_\text{max}^2 \leq \alpha$ and the optimization constraint $\| \hat \Theta \|_\text{max} \leq \alpha$. With this approach in mind, the following are caricatures of each of the lemmas:
\begin{enumerate}[(a)]
    \item Section~\ref{subsec:opnorm} (operator norm bound):
    \begin{align*}
        \| \nabla \cL(\Theta^*) \|_\text{op} \lesssim \alpha \sqrt{\frac{\log d}{dm}}.
    \end{align*}
    \item Section~\ref{subsec:rsc} (restricted strong convexity): 
    \begin{align*}
        &\cL(\Theta^* + \Delta) - \cL(\Theta^*) - \langle \nabla \cL(\Theta^*), \Delta \rangle  \\
        &\qquad \geq \frac{1}{d^2} \| \Delta \|_F^2 - c \alpha \| \Delta \|_\text{nuc} \sqrt{\frac{\log d}{dm}}.
    \end{align*}
    \item Section~\ref{subsec:decomp} (decomposability): 
    \begin{align*}
        \text{If $\lambda \geq 2\| \nabla \cL (\Theta^*) \|_\text{op}$, then $\| \Delta \|_\text{nuc} \lesssim \sqrt{r} \| \Delta \|_F$.}
    \end{align*}
\end{enumerate}
With these lemmas, we can first apply restricted strong convexity, reverse triangle, and Holder's inequality as before:
\begin{align*}
    0 &\geq \cL(\Theta^* + \Delta) - \cL(\Theta^*) + \lambda ( \| \Theta^* + \Delta \|_\text{nuc} - \| \Theta^* \|_\text{nuc} ) \\
    &\geq \frac{1}{d^2} \| \Delta \|_F^2 - c \alpha \| \Delta \|_\text{nuc} \sqrt{\frac{\log d}{dm}}  -\| \nabla \cL(\Theta^*) \|_\text{op} \| \Delta \|_\text{nuc} \\
    & - \lambda \| \Delta \|_\text{nuc}.
\end{align*}
Next, by our operator norm bound (Lemma~\ref{lem:opnormbound}) and our setting of $\lambda = 16 \alpha \sqrt{\frac{\log d}{dm}}$, we can replace the latter two terms with $c\alpha \| \Delta \|_\text{nuc} \sqrt{\frac{\log d}{dm}}$ and combine, producing
\begin{align*}
    &\geq \frac{1}{d^2} \| \Delta \|_F^2 - c' \alpha \| \Delta \|_\text{nuc} \sqrt{\frac{\log d}{dm}}.
\end{align*}
Finally, applying the bound $\| \Delta \|_\text{nuc} \lesssim \sqrt{r} \| \Delta \|_F$ (Lemma~{1} of \citet{negahban2012restricted}) and rearranging produces $\frac{1}{d} \| \Delta \|_F \lesssim \alpha \sqrt{\frac{rd\log d}{m}}$, as desired. 

\begin{figure*}[t]
\begin{center}
    \includegraphics[width=\linewidth]{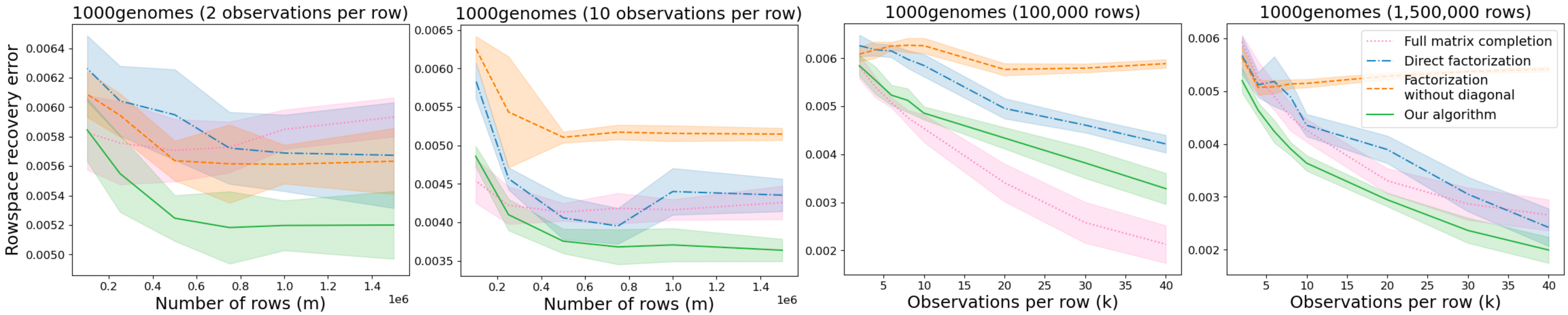}
\end{center}
\caption{\label{figure:1000genomes-plot} Rowspace recovery on the 1000genomes dataset, where from left to right the experiments are as follows: (a) we observe $k=2$ entries per row with the number of rows $m$ between $100{,}000$ and $1{,}500{,}000$, (b) $k=10$ with $m$ between $100{,}000$ and $1{,}500{,}000$, (c) $m = 100{,}000$ with $k$ between $2$ and $40$, and (d) $m = 1{,}500{,}000$ with $k$ between $2$ and $40$. The number of columns is fixed with $d = 250$ for all experiments. Our algorithm (in solid green) is most accurate for most parameter settings, but is outperformed by full matrix completion (in dotted pink) when $m$ is small and $k$ is large (plot (c)).}
\end{figure*}
 
\subsection{Operator norm bound}\label{subsec:opnormbound-main}
 It remains to discuss each of the lemmas, which we do briefly here and at length in the appendix. First, using our definition of the loss in Equation~\ref{eqn:loss-main}, we can compute the gradient of $\cL$ at $\Theta^*$ as follows (where recall that $E_{i,j}$ is the mask matrix with 1 at $(i,j)$ and 0 elsewhere):
\begin{align*}
\nabla \cL (\Theta^*) &= \frac{1}{m} \sum_{i=1}^m \frac{1}{2} (\Theta^*_{a(i), b(i)} - X_{i,a(i)}X_{i,b(i)}) E_{a(i), b(i)}  \\
&+ \frac{1}{m} \sum_{i=1}^m \frac{1}{2} (\Theta^*_{a(i), b(i)} - X_{i,a(i)}X_{i,b(i)}) E_{b(i), a(i)} \\
&+ \frac{1}{m} \sum_{i=1}^m \frac{1}{2} (\Theta^*_{a(i), a(i)} - X_{i,a(i)}^2) E_{a(i), a(i)} \\
&+ \frac{1}{m} \sum_{i=1}^m \frac{1}{2} (\Theta^*_{b(i), b(i)} - X_{i,b(i)}^2) E_{b(i), b(i)}.
\end{align*}
We wish to prove that the operator norm of this quantity is small, with high probability with respect to the randomly sampled indices $(a(i), b(i))$. In the example where the row factors $u_i$ were Gaussian, we had that $\EE_{u_i \sim Z} [X_{i,j_1}X_{i,j_2}] = \Theta^*_{j_1, j_2}$, turning each summand into a mask matrix multiplied by mean-zero noise. While that approach no longer holds here, we can write
\begin{align*}
    &X_{i,a(i)}X_{i,b(i)} - \Theta^*_{a(i), b(i)} = \\
        &\qquad v_{a(i)}^T\left(u_i u_i^T - \frac{1}{m} \sum_{i=1}^m u_i u_i^T\right)v_{b(i)}.
\end{align*}
At this point, we can apply our assumption that the mask is chosen independently of the underlying matrix, as it means that the expectation $\EE_{a(i), b(i)} \| \nabla \cL(\Theta^*) \|_\text{op}$ is invariant to permutations of the rows. Considering randomly permuted row factors brings us closer to the distributional case, which allows us to apply Radamacher symmetrization arguments (Section~\ref{subsec:opnorm}). The rest of the bound then proceeds via standard concentration arguments.

\subsection{Restricted strong convexity}
The other key lemma is restricted strong convexity; by the definition of the loss (Equation~\ref{eqn:loss-main}), we can first compute
\begin{align*}
    &\cL(\Theta^* + \Delta) - \cL(\Theta^*) - \langle \nabla \cL(\Theta^*), \Delta \rangle \\
    &\qquad = \frac{1}{m}\sum_{i=1}^m \frac{1}{2}[\Delta_{a(i), b(i)}^2 + \Delta_{b(i), a(i)}^2] \\
    &\qquad + \frac{1}{m}\sum_{i=1}^m \frac{1}{2}[\Delta_{a(i), a(i)}^2 + \Delta_{b(i), b(i)}^2].
\end{align*}
Recall that we wanted to show that this quantity is lower-bounded by $\frac{1}{d^2} \| \Delta \|_F^2$, minus some concentration terms. The first point to note is that each term has expectation
\begin{align*}
    \EE [\Delta_{a(i), b(i)}^2] &= \frac{1}{d(d-1)} \| P_\text{off-diag} (\Delta) \|_F^2,\\
    \EE [\Delta_{a(i), a(i)}^2] &= \frac{1}{d} \| P_\text{diag} (\Delta) \|_F^2,
\end{align*}
where $P_\text{diag}(\Delta)$ sets the off-diagonal terms of $\Delta$ to zero and $P_\text{off-diag}(\Delta) = \Delta - P_\text{diag}(\Delta)$ sets the diagonal terms to zero. Then, analyzing the off-diagonal terms and diagonal terms separately, we can first show concentration of the given sums around their expectations. The lemma then follows from proper treatment and recombination of the diagonal and off-diagonal terms.

\section{Experiments}\label{sec:experiments}

\subsection{Setup}

Finally, we evaluate our method on synthetic data and the 1000genomes dataset~\citep{fairly2019international}. The 1000genomes dataset contains fully sequenced chromosomes of 2354 subjects. Due to computational limitations, we first subsample to $m = 1\,500\,000$ rows (chosen from the first chromosome in order of mutation frequency) and the subjects to $d = 250$ columns (chosen randomly). We also reduce the number of rows further in some experiments to evaluate how error scales as a function of $m$. We then randomly sample between $k=2$ and $40$ observations per row and compare the estimated right-side singular vectors to the ground truth, produced by factoring the original fully sampled matrix. Specifically, let the ground truth SVD be given by $X = P\Sigma Q^T$. Then, we compute the error as\footnote{This minimization is known as the Procrustes problem~\citep{schonemann_procrustes_1996} and can be solved in closed form.}
\begin{align}\label{eqn:eval}
    \min_{R \in \RR^{r \times r}: R^TR = I_r} \| \hat Q R - Q \|_F^2.
\end{align}
The estimate $\hat Q$ is produced by the following algorithms:
\begin{enumerate}[(a)]
    \item Full matrix completion: perform rank $r$ matrix completion on $P_E(X)$ and compute the SVD of the result.
    \item Direct factorization: compute the rank $r$ SVD of the matrix $P_E(X)^TP_E(X)$.
    \item Factorization without diagonal~\citep{cai2021subspace}: compute the rank $r$ SVD of $P_\text{off-diag}(P_E(X)^TP_E(X))$.
    \item Our algorithm: perform rank $r$ matrix completion with respect to $\cL$ (Equation~\ref{eqn:loss-main}) and compute the SVD of the result. The loss function naturally generalizes to $k > 2$ via the formulation in Equation~\ref{eqn:loss-reinterpretation}.
\end{enumerate}
We implement vanilla matrix completion via the non-convex optimization\footnote{Note that the L2 regularizer in this objective corresponds to the likelihood under the Gaussian prior, which is the factor distribution that we use in our synthetic experiments. We choose $\lambda = 0.1$ using grid search over the set $(0, 0.001, 0.01, 0.1, 0.5, 1, 10)$.}
\begin{align*}
    \min_{\substack{U \in \RR^{m \times r}\\V \in \RR^{d \times r}}}\ &\frac{1}{|E|} \| P_E(UV^T) - P_E(X) \|_F^2 \\
    &+ \lambda \left( \frac{1}{m} \sum_{i=1}^m \| u_i \|_2^2 + \frac{1}{d} \sum_{i=1}^d \| v_i \|_2^2 \right).
\end{align*}
We also implement our method via the non-convex optimization $\min_{V \in \RR^{d \times r}} \cL(VV^T)$, where we fit a symmetric matrix and omit the L2 regularizer because the diagonal terms control the factor norms.

For the synthetic experiments, we sample $X = UV^T$ with $U \in \RR^{m \times r}$ and $V \in \RR^{d \times r}$ i.i.d.\ Gaussian $\mathcal{N}(0, r^{-1/2})$, rescaled to ensure that the expected norm of each entry is 1. We set the rank to $r = 25$ with $d = 100$ columns, while varying the number of rows $m$ and the observations per row $k$ depending on the experiment. 

We produce each of the plots by resampling the mask ten times for each parameter setting and plotting the mean plus or minus two standard deviations; see Section~\ref{subsec:hyperparams} for other implementation details and hyperparameters.

\subsection{Synthetic experiments}

\begin{figure}[t]
\begin{center}
    \includegraphics[width=0.7\linewidth]{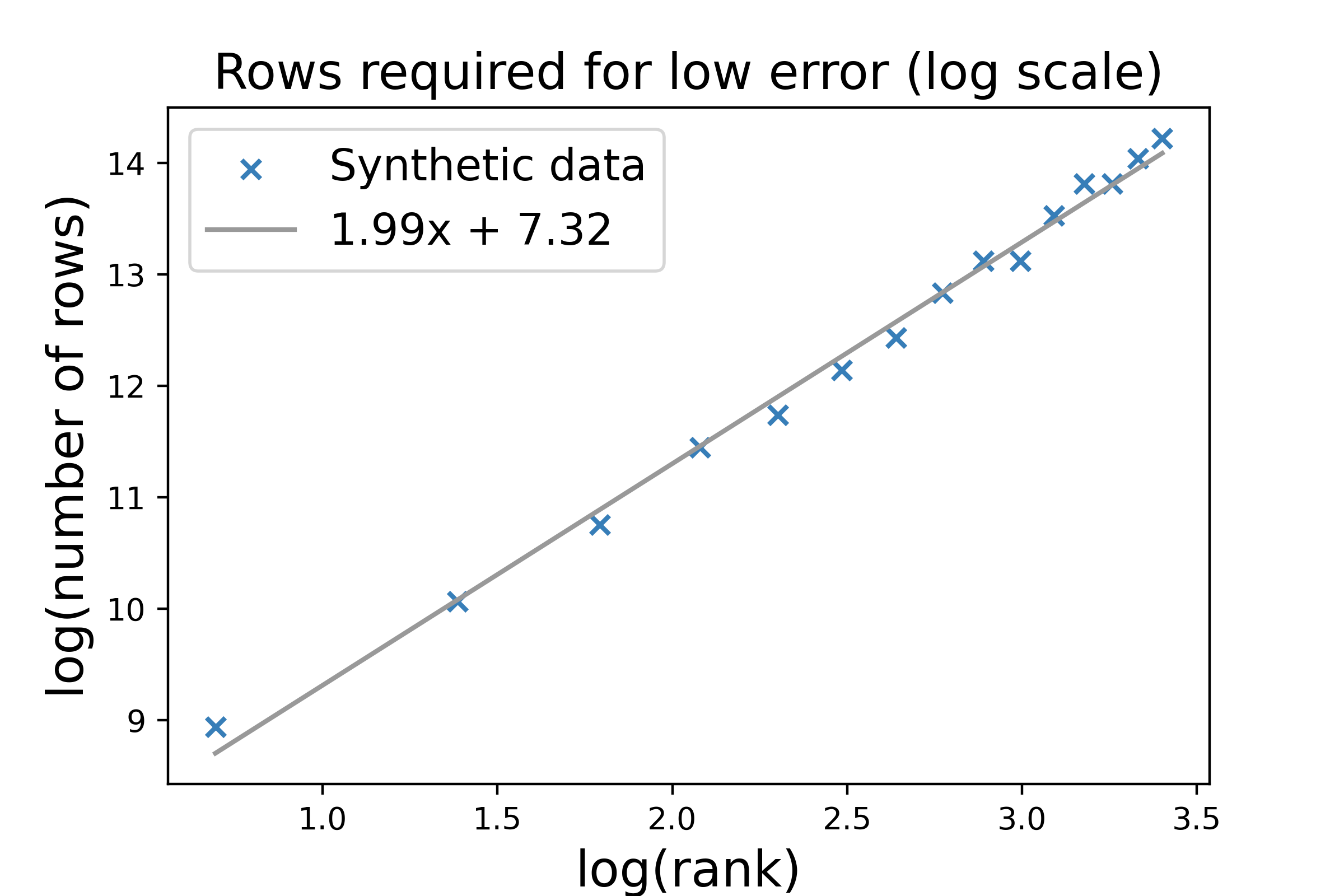}
\end{center}
\caption{\label{figure:rdep} A log plot of the row-rank dependence for our algorithm, where we plot the number of rows $m$ required to achieve low rowspace recovery error for each rank $r$ for i.i.d.\ Gaussian row and column factors, $k=2$ observations per row, and $d=200$ columns. The linear fit has slope almost exactly 2, suggesting a dependence of $m \propto r^2$ and confirming our derived bounds.}
\end{figure}

We first discuss the synthetic experiments, shown in Figure~\ref{figure:synthetic-plot}. From left to right, the experiments are as follows: (1) we sample $k=2$ observations per row and vary the number of rows $m$ from $10{,}000$ to $1{,}000{,}000$, (2) $k=30$ with $m$ between $100$ and $10{,}000$ (where fewer rows are necessary because $k$ is large), (3) $m = 10{,}000$ with $k$ between $2$ and $40$, and (4) $m = 500{,}000$ with $k$ between $2$ and $40$. 

In these experiments, our algorithm produces the most accurate rowspace estimates for all parameter ranges evaluated on. The algorithm is especially strong compared to other methods when $m$ is large or $k$ is small, which are the settings that motivate our study. Full matrix completion performs well when $m$ is small and $k$ is large, i.e.\ the setting with $X$ closer to square and with more observations, which is when we expect full completion to be feasible. However, even in this setting, our algorithm performs competitively with or outperforms full matrix completion.

\subsection{1000genomes experiments}
Next, we display similar plots for the 1000genomes dataset experiments, shown in Figure~\ref{figure:1000genomes-plot}, where from left to right the experiments are as follows: (1) we observe $k=2$ entries per row with the number of rows $m$ between $100{,}000$ and $1{,}500{,}000$, (2) $k=10$ with $m$ between $100{,}000$ and $1{,}500{,}000$, (3) $m = 100{,}000$ with $k$ between $2$ and $40$, and (4) $m = 1{,}500{,}000$ with $k$ between $2$ and $40$.\footnote{We set $r=10$ for each algorithm and compute the evaluation metric (Equation~\ref{eqn:eval}) with respect to the rank $10$ SVD of the original matrix. On the other hand, the ground truth TSNE plot is produced with respect to the full SVD of the original matrix.}

In these experiments, our algorithm is again most accurate for most parameter settings and is strongest relative to other methods when $m$ is large or $k$ is small. However, it is outperformed by full matrix completion when $m$ is small and $k$ is large, which is when we expect full completion to be possible. For the $m = 1{,}500{,}000$ setting, which is most representative of the ratio $m/d$ we might see in practice for this setting, our algorithm is most accurate for all values of $k$ evaluated on and is able to recover the rowspace reliably with as few as 5 observations per row.

Finally, we visualize the recovered column factors using TSNE~\citep{vandermaaten2008visualizing}, which projects the factors into two dimensions while attempting to preserve similarity structure (Figure~\ref{figure:tsne-1}). Visually, our algorithm recovers the most accurate factors by far and is almost identical to the ground truth. Factorization without the diagonal produces the next best estimates but is unable to separate the EUR, AMR, and SAS clusters. Meanwhile, full matrix completion is unable to recover coherent clusters because there are too few observations per row.

\subsection{Dependence on rank}
Finally, we perform synthetic experiments to verify the $r^2$ dependence in our derived sample complexity of $m = \Omega(r^2 d \log d)$. Specifically, for each rank $r$, we sample i.i.d.\ Gaussian factors with $d=200$ columns and $k=2$ observations per row, and we perform binary search over $m$ to achieve some target error.\footnote{For each candidate $m$ in the binary search, we average the loss over $20$ runs and accept if $\frac{1}{r} \| \hat Q R - Q \|_F^2 \in 0.1 \pm 0.02$. The search starts with a range of $m \in (0, 4e6)$.} The result is shown in Figure~\ref{figure:rdep}, where we plot $\log m$ versus $\log r$. The points lie almost exactly on a line with slope 2, suggesting a dependence of $m \propto r^2$ and confirming our derived bounds.

\section{Related Work}\label{sec:related}

While low-rank matrix completion and factorization have a long history of research, here we touch on just a few threads of work most directly related to our paper. 

\textbf{Matrix completion.} Matrix completion is the problem of estimating a low-rank matrix after observing a subset of its entries, and a variety of methods have been proposed and analyzed for this problem, including nuclear norm minimization~\citep{candes2009exact,candes2010matrix}, SVD with trimming~\citep{keshavan2009matrix}, alternating minimization~\citep{hardt2014understanding}, and non-convex gradient descent~\citep{ge2016matrix,jin2016provable}. To rule out certain matrices whose recovery is impossible, these papers propose various kinds of incoherence assumptions: for example, \citet{candes2009exact} make assumptions on the leverage scores of $X$, while \citet{negahban2012restricted} make assumptions about the ``spikiness'' of $X$. We adopt the assumptions and analysis framework of \citet{negahban2012restricted}, who prove additive Frobenius error bounds under the assumption that the maximum entry of the underlying matrix is bounded.

\textbf{Unbalanced noisy matrix factorization.} A recent series of works explores the problem of noisy matrix factorization for matrices with high aspect ratio (i.e.\ many more rows than columns). In this problem setting, the matrix is fully observed but with entry-wise additive noise. \citet{feldman2021spiked}, \citet{donoho2022optimal}, and \citet{montanari2022fundamental} study the asymptotics of this problem, characterizing the Bayes optimal error of recovering the singular vectors as the number of rows and columns $m, d \to \infty$. Broadly speaking, in contrast with past works which took $m, d \to \infty$ with the ratios $d/m$ and $m/d$ remaining bounded, these papers consider cases where $d/m \to \infty$ and $d/m \to 0$, finding regimes where recovery of the left singular values is possible but not the right singular vectors (and vice versa). While their setting and results are very different from ours, they share the commonality of studying cases where only ``one-sided'' recovery is possible.

\textbf{Subspace and covariance estimation from partial observations.} The papers most directly related to our setting are \citet{lounici2014high}, \citet{gonen2016subspace}, \citet{montanari-sun-2018}, and \citet{cai2021subspace}. \citet{gonen2016subspace} consider the problem of subspace estimation from partial observations: given $m$ partially observed vectors of dimension $d$, sampled i.i.d.\ from a bounded distribution, their goal is to recover the rank-$r$ subspace that those vectors lie in. As their algorithm, they perform factorization on $P_E(X)^TP_E(X)$ with the diagonal rescaled. However, they make no incoherence assumptions, making matrix completion inapplicable. Therefore, they find that a sample size of $m = \Omega((d/k)^2 r)$ is both necessary and sufficient to recover the subspace (where $k \geq 2$ is the average number of observations per vector). \citet{lounici2014high} study a similar setting, but with the weaker assumption that the vectors are sampled from a sub-Gaussian distribution, and they prove a similar sample complexity of $m = \Omega((d/k)^2 r \log d)$.

\citet{montanari-sun-2018} and \citet{cai2021subspace} also study subspace estimation from partial observations. Similar to the above papers, both algorithms involve factoring $P_E(X)^TP_E(X)$ with the diagonal rescaled, but they adopt incoherence assumptions to establish sample complexity bounds no longer quadratic in $d$. As their sampling distribution, they assume $n$ observations uniformly chosen from the $md$ entries, and they both prove similar sample complexities of $n = \Omega(r \sqrt{dm}\ {\rm polylog}(dm))$ (please see the original papers for the full results). For $k$ observations per row on average, the number of rows required then becomes $m = \Omega(r^2 (d/k)\ {\rm polylog}(d))$.

One difference in setting between our paper and the aforementioned papers is that they focus on subspace estimation, so it suffices to show that $P_E(X)^TP_E(X)$ (after rescaling) is close to $X^TX$ in operator norm. In contrast, because we are also interested in completing $X^TX$ and recovering the column factors, we show error bounds in Frobenius norm, which requires more accurate estimation of $X^TX$.

The other main difference between our paper and the aforementioned papers is that they consider the setting where each entry of $X$ is observed independently with probability $p$ (or equivalently, that the $n$ observations are uniform over the matrix). Under this model, even if $p$ is small enough such that there are two observations per row on average, some rows might still have larger numbers of observations. In contrast, we show that the column factors can be recovered even if \textit{all} of the rows have only two observations. From a theoretical standpoint, this setting is strictly harder than the Bernoulli observation setting because given the latter, we can keep the rows with at least two observations (which is satisfied by roughly $3/4$ of the rows for $d \gg 2$), subsample to two per row, and use our analysis to produce the same sample complexity (up to constant factors).

\section{Conclusion and Future Directions}\label{sec:conclusion}

One limitation of our result is that it only applies to two observations per row; therefore, a fruitful direction could involve extending it to more general cases, like $k$ observations per row or other sampling patterns. Empirically, we hope that our paper can inspire work on datasets that were previously too sparsely annotated for full matrix completion, but might be amenable to our algorithm.

\section*{Acknowledgements}
We thank John Thickstun, Tae Kyu Kim, and anonymous reviewers for their helpful comments and discussion. Gregory Valiant is supported by a Simons Investigator Award, and NSF awards CCF-1704417 and CCF-1813049. Steven Cao is supported by the NSF Graduate Research Fellowship. Experiments were run on the Stanford NLP cluster.

\bibliography{main}
\bibliographystyle{icml2023}

\newpage
\onecolumn
\appendix

\section{Proofs}

\subsection{Notation}
Letting $A$ and $B$ represent an arbitrary matrices, we'll use the following notation:
\begin{itemize}
    \item $\|A\|_\text{max}$ -- maximum of the absolute values of the entries of $A$
    \item $\|A\|_\text{op}$ -- operator norm
    \item $\|A\|_\text{nuc}$ -- nuclear norm
    \item $\|A\|_F$ -- Frobenius norm
    \item $\langle A, B \rangle = \text{tr}(A^TB)$ -- matrix inner product
    \item For a square matrix $C \in \RR^{d \times d}$, we'll use $P_\text{diag}(C)$ to represent the matrix $C$ with the off-diagonal terms set to zero, and $P_\text{off-diag}(C) = C - P_\text{diag}(C)$ to represent the matrix $C$ with the diagonal terms set to zero.
    \item We'll use $E_{i,j} \in \RR^{d \times d}$ to represent the matrix with 1 in entry $(i, j)$ and 0 elsewhere, and $\tilde E_{i, j}$ to refer to the symmetric mask $\frac{1}{2}(E_{i,j} + E_{j, i})$. 
    \item $f \lesssim g$ will denote that there exists a universal constant $c$ such that $f \leq cg$. 
    \item $[d]$ will denote the set $\{1, 2, ..., d \}$.
\end{itemize}
Throughout the proofs, when there are long chains of inequalities, we will sometimes box the terms that change from line to line for ease of reading.

\subsection{Setup and main result}\label{subsec:setup}

The problem setup is as follows: from a rank-$r$ matrix $X = UV^T \in \RR^{m \times d}$ with $U \in \RR^{m \times r}$ and $V \in \RR^{d \times r}$, we randomly observe two entries per row, which we can represent as indices $(a(1), b(1)), ..., (a(m), b(m))$ drawn i.i.d.\ uniformly from the set $\{(i,j) : i, j \in [d],\ i \neq j\}$. We wish to estimate the matrix
\begin{align*}
    \Theta^* &= \frac{1}{m}X^TX = V \bar S V^T \\
    \bar S &= \frac{1}{m} U^TU = \frac{1}{m} \sum_{i=1}^m u_i u_i^T \in \RR^{r \times r},
\end{align*}
where $u_i \in \RR^{r}$ is the $i$th row of $U$. Our estimator for $\Theta^*$ minimizes a squared loss with a nuclear norm regularizer, along with constraints on the maximum off-diagonal and diagonal entries, and is given as follows:
\begin{align}\label{eqn:program}
    \hat \Theta \in \argmin_{\| \Theta \|_\text{max} \leq \alpha} \cL(\Theta) + \lambda \| \Theta \|_\text{nuc},
\end{align}
where the squared loss $\cL(\Theta)$ is given by
\begin{align}\label{eqn:loss}
    \cL(\Theta) &=  \frac{1}{4m} \sum_{i=1}^m (\Theta_{a(i), b(i)} - X_{i, a(i)} X_{i, b(i)})^2 + (\Theta_{b(i), a(i)} - X_{i, b(i)} X_{i, a(i)})^2\nonumber \\
    &+ \frac{1}{4m} \sum_{i=1}^m (\Theta_{a(i), a(i)} - X_{i, a(i)}^2)^2 + (\Theta_{b(i), b(i)} - X_{i, b(i)}^2)^2.
\end{align}
Given the assumption that $\| X \|_\text{max}^2 \leq \alpha$, we will prove the following error bound:
\begin{thm}\label{thm:main}
Let $\hat \Theta$ be the solution of the optimization problem defined in Equation~\ref{eqn:program}, where $\lambda$ is set to $16 \alpha \sqrt{\frac{ \log d + \delta}{dm}}$. Also, suppose that $X$ is rank $r$ with $\| X \|_\text{max}^2 \leq \alpha$, and $m \geq d(\log d + \delta)$. Then, with probability $\geq 1 - 3e^{-\delta}$, we have that
\begin{align*}
    \frac{1}{d^2} \| \hat \Theta - \Theta^* \|_F^2 \lesssim \alpha^2 \frac{rd(\log d + \delta)}{m}.
\end{align*}
\end{thm}

\subsection{Rowspace recovery}

From Theorem~\ref{thm:main}, we can first derive a rowspace recovery result, where our goal is to estimate the subspace spanned by the rows of $X$. In particular, the rowspace of $X$ is equal to the rowspace of $\Theta^* = \frac{1}{m} X^TX$, and we have an error bound for our estimator $\hat \Theta$. Therefore, we can use the rowspace of $\hat \Theta$ as our estimator and use standard perturbation theory to bound the rowspace estimation error. Specifically, letting the rank SVD of $\Theta^*$ be given by $Q \Lambda Q^T$ for $Q \in \RR^{d \times r}$, we can think of rowspace recovery as estimating the right-side singular vectors $Q$ up to rotation. Then, we can define rowspace recovery error as $\min_{R \in \RR^{r \times r}: R^TR = I_r} \| \hat Q R - Q \|_F^2$ and produce the following error bound:
\begin{cor}[Right-side singular vector recovery]\label{cor:rowspace}
    Let $\hat \Theta$ be the solution of the optimization problem defined in Equation~\ref{eqn:program}, where $\lambda$ is set to $16 \alpha \sqrt{\frac{ \log d + \delta}{dm}}$, and let $\hat Q \in \RR^{d \times r}$ be the top $r$ singular vectors of $\hat \Theta$. Also, suppose that $X$ is rank $r$ with $\| X \|_\text{max}^2 \leq \alpha$ and $m \geq d(\log d + \delta)$, and let $Q \in \RR^{d \times r}$ denote the right-side singular vectors of $X$. Then, letting $\sigma_r$ be the $r$th singular value of $\Theta^* = \frac{1}{m}X^TX$, we have that with probability $\geq 1 - 3e^{-\delta}$,
    \begin{align*}
        \min_{R \in \RR^{r \times r}: R^TR = I_r} \| \hat Q R - Q \|_F^2 \lesssim \left(\frac{d \alpha}{\sigma_r}\right)^2 \frac{rd(\log d + \delta)}{m}.
    \end{align*}
\end{cor}
Note that while the additive error in Theorem~\ref{thm:main} scales with the magnitude of $X$, the rowspace recovery error is scale-invariant. Therefore, while the bound in Theorem~\ref{thm:main} scales with $\alpha$, the bound in Corollary~\ref{cor:rowspace} is in terms of the quantity $\frac{d\alpha}{\sigma_r}$. This quantity can be thought of as capturing both the incoherence and condition number of $X$: for example, it is large for ``spiky'' matrices, where rowspace recovery is impossible (e.g. $dm$ for the matrix with 1 in a single entry and 0 otherwise), and it is small for ``incoherent'' matrices (e.g.\ $1$ for the all-ones matrix). The corollary directly follows from Theorem 2 of \citet{yu2015useful}, a variant of the Davis-Kahan theorem~\citep{davis-kahan-70}:
\begin{lem}[Theorem 2 of \citet{yu2015useful}]
    Let $A, \hat A \in \RR^{d \times d}$ be symmetric matrices with eigenvalues $\lambda_1 \geq ... \geq \lambda_d$ and $\hat \lambda_1 \geq ... \geq \hat \lambda_d$. Fixing some $1 \leq r \leq d$, suppose that $\lambda_r - \lambda_{r+1} > 0$, where $\lambda_{d+1} := -\infty$. Let $V = (v_1, ..., v_r) \in \RR^{d \times r}$ and $\hat V = (\hat v_1, ..., \hat v_r) \in \RR^{d \times r}$ have orthonormal columns satisfying $A v_i = \lambda_i v_i$ and $\hat A \hat v_i = \hat \lambda_i \hat v_i$. Then,
    \begin{align*}
        \min_{R \in \RR^{r \times r}: R^T R = I_r} \| \hat V R - V \|_F \leq \frac{2^{3/2} \| \hat A - A \|_F}{\lambda_r - \lambda_{r+1}}.
    \end{align*}
\end{lem}
\begin{proof}[Proof of Corollary~\ref{cor:rowspace}]
    First, letting the SVD of $X$ be given by $P \Sigma Q^T$, we have that the SVD of $\Theta^*$ is given by $Q \Lambda Q^T$ (for $\Lambda = \frac{1}{m} \Sigma^T\Sigma \in \RR^{r \times r}$). Then, by Theorem~{2} of \citet{yu2015useful}, we have that
    \begin{align*}
        \min_{R \in \RR^{r \times r}: R^TR = I_r} \| \hat Q R - Q \|_F^2 \leq 8 \frac{\| \hat \Theta - \Theta^* \|_F^2}{\sigma_r(\Theta^*)^2}.
    \end{align*}
    The corollary then immediately follows from Theorem~\ref{thm:main}.
\end{proof}

\subsection{Column factor recovery}

In Corollary~\ref{cor:rowspace}, a $\sigma_r$ factor appears in the denominator  because the algorithm is tasked with recovering all $r$ singular directions, even if some singular directions have low weight (i.e.\ $\sigma_r$ is small). However, in many applications, we often care only about recovering the singular directions with high weight, as the low weight singular directions have little effect on the data. Therefore, from Theorem~\ref{thm:main} we can also derive a column factor recovery result, which can be thought of as a weighted version of the rowspace recovery result. In particular, our goal here is to recover $Q \Lambda^{1/2} \in \RR^{d \times r}$ up to rotation, which can be thought of as containing column factors, or $r$-dimensional vector representations for each column in $X$. Then, using the top $r$ singular values and vectors of $\hat \Theta$ for our estimate, we can produce the following error bound:
\begin{cor}[Column factor recovery]\label{cor:colfactor}
    Let $\hat \Theta$ be the solution of the optimization problem defined in Equation~\ref{eqn:program}, where $\lambda$ is set to $16 \alpha \sqrt{\frac{ \log d + \delta}{dm}}$, and let $\hat Q \in \RR^{d \times r}$ and $\hat \Lambda \in \RR^{r \times r}$ be the top $r$ singular vectors and singular values of $\hat \Theta$. Also, suppose that $X$ is rank $r$ with $\| X \|_\text{max}^2 \leq \alpha$ and $m \geq d(\log d + \delta)$, and let $\Theta^* = Q \Lambda Q^T$ be the SVD of $\Theta^* = \frac{1}{m}X^TX$, where $Q \in \RR^{d \times r}$ and $\Lambda \in \RR^{r \times r}$. Then, we have that with probability $\geq 1 - 3e^{-\delta}$,
    \begin{align*}
        \min_{R \in \RR^{r\times r}: R^TR = I_r} \frac{1}{d} \| \hat Q \hat \Lambda^{1/2} R - Q \Lambda^{1/2} \|_F^2 &\lesssim \alpha \sqrt{\frac{r^2 d(\log d + \delta)}{m}}.
    \end{align*}
\end{cor}
Note that unlike Corollary~\ref{cor:rowspace}, this bound does not depend on $\sigma_r$ because the singular vectors are weighted by their corresponding singular values. The proof uses Powers-St{\o}rmer~\citep{Powers1970} and proceeds as follows:
\begin{lem}[Powers-St{\o}rmer]
    For positive semidefinite matrices $A$ and $B$, we have that
    \begin{align*}
        \| A - B \|_F^2 \leq \| A^2 - B^2 \|_\text{nuc}.
    \end{align*}
\end{lem}
\begin{proof}[Proof of Corollary~\ref{cor:colfactor}]
    First, by the Powers-St{\o}rmer inequality~\citep{Powers1970}, we have that
    \begin{align*}
        \| \hat Q \hat \Lambda^{1/2} \hat Q^T - Q \Lambda^{1/2} Q^T \|_F^2 \leq \| \hat \Theta_r - \Theta^* \|_\text{nuc},
    \end{align*}
    where $\hat \Theta_r = \hat Q \hat \Lambda \hat Q^T$ is the rank-$r$ truncated version of $\hat \Theta$. Next, we have that
    \begin{align*}
        \| \hat \Theta_r - \Theta^* \|_\text{nuc} &\overset{(i)}{\leq} \sqrt{2r} \| \hat \Theta_r - \Theta^* \|_F \\
        &\overset{(ii)}{\leq} \sqrt{2r} (\| \hat \Theta - \hat \Theta_r \|_F + \| \hat \Theta - \Theta^* \|_F) \\
        &\overset{(iii)}{\leq} 2 \sqrt{2r} \| \hat \Theta - \Theta^* \|_F \\
        &\overset{(iv)}{\lesssim} \alpha \sqrt{r} d \sqrt{\frac{r d (\log d + \delta)}{m}}.
    \end{align*}
    where (i) follows from $\hat \Theta_r - \Theta^*$ being at most rank $2r$, (ii) from triangle, (iii) from the Eckart-Young-Mirsky theorem~\citep{Eckart1936,mirsky1960}, which states that $\hat \Theta_r$ is the closest rank-$r$ matrix to $\hat \Theta$ in any unitarily invariant norm, and (iv) from applying Theorem~\ref{thm:main}. Then, it suffices to show that 
    \begin{align*}
        \min_{R \in \RR^{r\times r}: R^TR = I_r} \| \hat Q \hat \Lambda^{1/2} R - Q \Lambda^{1/2} \|_F^2 &\leq \| \hat Q \hat \Lambda^{1/2} \hat Q^T - Q \Lambda^{1/2} Q^T \|_F^2.
    \end{align*}
    To show this inequality, we can choose a particular rotation $R$ as follows:
    \begin{align*}
        \| \hat Q \hat \Lambda^{1/2} R - Q \Lambda^{1/2} \|_F^2 &= \text{tr}(\Lambda) + \text{tr}(\hat \Lambda) - 2 \text{tr}(R^T \hat \Lambda^{1/2} \hat Q^TQ \Lambda^{1/2}) \\
        &\overset{(i)}{=} \text{tr}(\Lambda) + \text{tr}(\hat \Lambda) - 2 \| \hat \Lambda^{1/2} \hat Q^TQ \Lambda^{1/2} \|_\text{nuc} \\
        &\overset{(ii)}{=} \text{tr}(\Lambda) + \text{tr}(\hat \Lambda) - 2 \| \hat Q \hat \Lambda^{1/2} \hat Q^TQ \Lambda^{1/2} Q \|_\text{nuc} \\
        &\overset{(iii)}{\leq} \text{tr}(\Lambda) + \text{tr}(\hat \Lambda) - 2 \text{tr}( \hat Q \hat \Lambda^{1/2} \hat Q^TQ \Lambda^{1/2} Q ) \\
        &= \| \hat Q \hat \Lambda^{1/2} \hat Q^T - Q \Lambda^{1/2} Q^T \|_F^2
    \end{align*}
    where (i) follows from choosing $R = AB^T$ for $A$ and $B$ given by the SVD $\hat \Lambda^{1/2} \hat Q^TQ \Lambda^{1/2} = A S B^T$, (ii) follows from the fact that $\| C \|_\text{nuc} = \| Q_1 C Q_2^T \|_\text{nuc}$ for any $Q_1, Q_2$ such that $Q_1^TQ_1 = I$ and $Q_2^TQ_2 = I$, and (iii) follows from the fact that $\text{tr}(C) \leq |\text{tr}(C)| \leq \| C \|_\text{nuc}$ for any matrix $C$, proving the desired result.
\end{proof}

\subsection{Proof outline}
In this section, we outline the proof of Theorem~\ref{thm:main}. Our analysis uses restricted strong convexity arguments, as described in~\citet{negahban2012restricted}, \citet{negahban2012unified}, and \citet{wainwright2019high}. For completeness, we reproduce parts of their analysis; such lemmas will also be marked with their source. The proof proceeds as follows: letting $\Delta$ denote the error $\hat \Theta - \Theta^*$, by the optimality of $\hat \Theta$ we can write
\begin{align*}
    0 &\geq \cL(\Theta^* + \Delta) - \cL(\Theta^*) + \lambda ( \| \Theta^* + \Delta \|_\text{nuc} - \| \Theta^* \|_\text{nuc} ) \\
    &\geq \cL(\Theta^* + \Delta) - \cL(\Theta^*) - \lambda \| \Delta \|_\text{nuc},
\end{align*}
where the second line follows from reverse triangle. If $\cL$ were strongly convex with parameter $\tau$, then we would have $0 \geq \cL(\Theta^* + \Delta) - \cL(\Theta^*) - \lambda \| \Delta \|_\text{nuc} \geq \langle \nabla \cL(\Theta^*), \Delta \rangle + \tau \| \Delta \|_F^2  - \lambda \| \Delta \|_\text{nuc}  \geq - \| \nabla \cL(\Theta^*) \|_\text{op} \| \Delta \|_\text{nuc} + \tau \| \Delta \|_F^2  - \lambda \| \Delta \|_\text{nuc} $, where the second inequality follows from strong convexity and the third from Holder's inequality. Then, bounding $\| \nabla \cL(\Theta^*) \|_\text{op}$ (Lemma~\ref{lem:opnormbound}) and $\| \Delta \|_\text{nuc}$ (Lemma~\ref{lem:decomp}) and rearranging would result in a bound for the error $\| \Delta \|_F^2$. However, $\cL$ is not strongly convex: if we do not restrict to low-rank matrices, then there are multiple matrices that agree with the incomplete observations. Therefore, we'll instead show a form of restricted strong convexity: in particular, we'll show (in Lemma~\ref{lem:rsc}) that the quantity $\cL(\Theta^* + \Delta) - \cL(\Theta^*) - \langle \nabla \cL(\Theta^*), \Delta \rangle$ concentrates around $\frac{1}{d^2} \| \Delta \|_F^2$, but with deviation terms depending on $\| \Delta \|_\text{nuc}$ and $\| \Delta \|_\text{max}$. Therefore, recovery will depend on $\| \Delta \|_\text{nuc}$ being small, which follows from the condition that the regularization strength $\lambda$ is large enough (Lemma~\ref{lem:decomp}), and the entries of $\Delta$ being bounded, which follows by our assumption that $\| X \|_\text{max}^2 \leq \alpha$. 

We'll conclude this outline by providing caricatures of each of the lemmas and showing how they come together to produce the desired bound.
\begin{enumerate}[(a)]
    \item Section~\ref{subsec:opnorm} (operator norm bound):
    \begin{align*}
        \| \nabla \cL(\Theta^*) \|_\text{op} \lesssim \alpha \sqrt{\frac{\log d}{dm}}.
    \end{align*}
    \item Section~\ref{subsec:rsc} (restricted strong convexity): 
    \begin{align*}
        \cL(\Theta^* + \Delta) - \cL(\Theta^*) - \langle \nabla \cL(\Theta^*), \Delta \rangle \geq  \frac{1}{d^2} \| \Delta \|_F^2 - c \alpha \| \Delta \|_\text{nuc} \sqrt{\frac{\log d}{dm}}.
    \end{align*}
    \item Section~\ref{subsec:decomp} (decomposability): 
    \begin{align*}
        \text{If $\lambda \geq 2\| \nabla \cL (\Theta^*) \|_\text{op}$, then $\| \Delta \|_\text{nuc} \lesssim \sqrt{r} \| \Delta \|_F$.}
    \end{align*}
\end{enumerate}
Then, from these lemma caricatures, we can first apply restricted strong convexity as follows:
\begin{align*}
    0 &\geq \cL(\Theta^* + \Delta) - \cL(\Theta^*) - \lambda \| \Delta \|_\text{nuc} \\
    &\overset{(i)}{\geq} \boxed{ \langle \nabla \cL(\Theta^*), \Delta \rangle + \frac{1}{d^2} \| \Delta \|_F^2 - c \alpha \| \Delta \|_\text{nuc} \sqrt{\frac{\log d}{dm}} } - \lambda \| \Delta \|_\text{nuc} \\
    &\overset{(ii)}{\geq} \boxed{ -\| \nabla \cL(\Theta^*) \|_\text{op} \| \Delta \|_\text{nuc} } + \frac{1}{d^2} \| \Delta \|_F^2 - c \alpha \| \Delta \|_\text{nuc} \sqrt{\frac{\log d}{dm}} - \lambda \| \Delta \|_\text{nuc} \\
    &\overset{(iii)}{\geq} \boxed{ -\frac{\lambda}{2} \| \Delta \|_\text{nuc} } + \frac{1}{d^2} \| \Delta \|_F^2 - c \alpha \| \Delta \|_\text{nuc} \sqrt{\frac{\log d}{dm}} - \lambda \| \Delta \|_\text{nuc}
\end{align*}
where (i) follows from Lemma~\ref{lem:rsc} (restricted strong convexity), (ii) from Holder's inequality, and (iii) from choosing $\lambda \geq 2 \| \nabla \cL (\Theta^*) \|_\text{op}$. Next, we can combine the terms with $\lambda$ and use our upper bound on $\| \Delta \|_\text{nuc}$, producing
\begin{align*}
    &= \frac{1}{d^2} \| \Delta \|_F^2 - \| \Delta \|_\text{nuc}\left( c \alpha \sqrt{\frac{\log d}{dm}} + \frac{3}{2} \lambda \right) \\
    &\overset{(iv)}{=} \frac{1}{d^2} \| \Delta \|_F^2 - \| \Delta \|_\text{nuc} \left( c \alpha \sqrt{\frac{\log d}{dm}} + \boxed{ c' \alpha \sqrt{\frac{\log d}{dm}}} \right) \\
    &\overset{(v)}{\geq} \frac{1}{d^2} \| \Delta \|_F^2 - \boxed{ \sqrt{r} \| \Delta \|_F } c'' \alpha \sqrt{\frac{\log d}{dm}}
\end{align*}
 where (iv) follows from choosing $\lambda = O(\| \nabla \cL (\Theta^*) \|_\text{op})$ and Lemma~\ref{lem:opnormbound} (operator norm bound), and (v) from Lemma~\ref{lem:decomp} (nuclear norm bound). Finally, rearranging produces $\frac{1}{d} \| \Delta \|_F \lesssim \alpha \sqrt{\frac{rd\log d}{m}}$ as desired. Note that this calculation is not a proof; please see Section~\ref{subsec:thmproof} for the full proof.

\subsection{Operator norm bound}\label{subsec:opnorm}

In this section, we use concentration arguments to upper bound the operator norm $\| \nabla \cL (\Theta^*) \|_\text{op}$ with high probability (where the randomness is over the sampled indices). Recall that $\cL$ is defined in Equation~\ref{eqn:loss}; a quick calculation reveals that
\begin{align*}
\nabla \cL (\Theta^*) &= \frac{1}{m} \sum_{i=1}^m (\Theta^*_{a(i), b(i)} - X_{i,a(i)}X_{i,b(i)}) \tilde E_{a(i), b(i)}  \\
&+ \frac{1}{2m} \sum_{i=1}^m (\Theta^*_{a(i), a(i)} - X_{i,a(i)}^2) E_{a(i), a(i)} \\
&+ \frac{1}{2m} \sum_{i=1}^m (\Theta^*_{b(i), b(i)} - X_{i,b(i)}^2) E_{b(i), b(i)},
\end{align*}
where we use $\tilde E_{a(i), b(i)}$ to denote the symmetric mask $\frac{1}{2}(E_{a(i), b(i)} + E_{b(i), a(i)})$.

\begin{lem}\label{lem:opnormbound} Given matrices $X = UV^T \in \RR^{m \times d}$ and $\Theta^* = \frac{1}{m} X^TX \in \RR^{d \times d}$, suppose that $X$ is bounded by $\| X \|_\text{max}^2 \leq \alpha$. Also, let $(a(1), b(1)), ..., (a(m), b(m))$ denote indices sampled i.i.d.\ uniformly from the set $\{(i,j) : i, j \in [d], i \neq j\}$. Then, for $\cL$ defined in Equation~\ref{eqn:loss}, we have that
\begin{align*}
    \| \nabla \cL(\Theta^*) \|_\text{op} \leq 8 \alpha \sqrt{\frac{\log d + \delta}{dm}}
\end{align*}
with probability $\geq 1 - e^{-\delta}$, for $m \geq d(\log d + \delta)$.
\end{lem}
\begin{proof}
    We'll first divide the bound into three parts, such that by triangle we have that
    \begin{align*}
        \| \nabla \cL (\Theta^*) \|_\text{op} &\leq \left\| \frac{1}{m} \sum_{i=1}^m (\Theta^*_{a(i), b(i)} - X_{i,a(i)}X_{i,b(i)}) \tilde E_{a(i), b(i)}\right\|_\text{op} \\
        &+ \left\| \frac{1}{2m} \sum_{i=1}^m (\Theta^*_{a(i), a(i)} - X_{i,a(i)}^2) E_{a(i), a(i)} \right\|_\text{op} \\
        &+ \left\| \frac{1}{2m} \sum_{i=1}^m (\Theta^*_{b(i), b(i)} - X_{i,b(i)}^2) E_{b(i), b(i)} \right\|_\text{op}.
    \end{align*}
    Then, we can bound each part separately with high probability and then use the union bound to bound their sum. The bound for each term will proceed as follows: first, writing the sum as $\| \frac{1}{m} \sum_{i=1}^m Q_i \|_\text{op}$ for ease of notation, we can use Markov's inequality to produce the Chernoff bound
    \begin{align*}
        \PP\left( \left\| \frac{1}{m} \sum_{i=1}^m Q_i \right\|_\text{op} \geq t \right) \leq \EE \left[\exp\left\{ \xi \left\| \sum_{i=1}^m Q_i \right\|_\text{op} \right\} \right] e^{-\xi m t}.
    \end{align*}
    Next, we'll use symmetrization to bound 
    \begin{align*}
        \EE \left[\exp\left\{ \xi \left\| \sum_{i=1}^m Q_i \right\|_\text{op} \right\} \right] \leq  \EE \left[\exp\left\{2 \xi \left\| \sum_{i=1}^m \eps_i \tilde Q_i \right\|_\text{op} \right\} \right],
    \end{align*}
    where $\eps_i$ are i.i.d.\ Radamacher random variables (i.e.\ uniform over the set $\{-1, +1\}$). Finally, we can bound the moments $\EE (\eps_i \tilde Q_i)^{2n}$ to bound this expectation, leading to a matrix Bernstein bound.

    \textbf{Radamacher symmetrization}: we'll first apply the symmetrization argument to the first term; the other two terms proceed similarly. First, note that by the definitions of $X = UV^T$ and $\Theta^* = \frac{1}{m} X^TX$ we can write
    \begin{align*}
        X_{i,a(i)}X_{i,b(i)} - \Theta^*_{a(i), b(i)} = \left\langle E_{a(i), b(i)}, V\left(u_i u_i^T - \frac{1}{m} \sum_{i=1}^m u_i u_i^T\right)V^T \right\rangle,
    \end{align*}
    where $u_i$ denotes the $i$th row of $U$. Substituting this expression into the expectation $\EE \left[\exp\left\{ \xi \left\| \sum_{i=1}^m Q_i \right\|_\text{op} \right\} \right]$, we have
    \begin{align*}
        &\EE \left[\exp\left\{ \xi \left\| \sum_{i=1}^m (\Theta^*_{a(i), b(i)} - X_{i,a(i)}X_{i,b(i)}) \tilde E_{a(i),b(i)} \right\|_\text{op} \right\} \right] \\
        &\qquad = \EE \left[\exp\left\{ \xi \left\| \sum_{i=1}^m \left\langle E_{a(i), b(i)}, V\left(u_i u_i^T - \frac{1}{m} \sum_{i=1}^m u_i u_i^T\right)V^T \right\rangle \tilde E_{a(i),b(i)} \right\|_\text{op} \right\} \right].
    \end{align*}
    Next, because the random indices are drawn i.i.d., this expectation is invariant to the $u_i$'s being permuted with each other (i.e.\ sampling a random permutation $\sigma$ and setting $u_i' = u_{\sigma(i)}$). Therefore, we can take an expectation with respect to sampling a random permutation $\sigma$ while also replacing $\frac{1}{m} \sum_{i=1}^m u_i u_i^T = \EE_{\tilde \sigma} u_{\tilde \sigma(i)} u_{\tilde \sigma(i)}^T$, resulting in
    \begin{align*}
        &= \EE_{a(i), b(i)} \boxed{\EE_\sigma} \left[\exp\left\{ \xi \left\| \sum_{i=1}^m \langle E_{a(i), b(i)}, V  \left( \boxed{u_{\sigma(i)} u_{\sigma(i)}^T - \EE_{\tilde \sigma} [u_{\tilde \sigma(i)} u_{\tilde \sigma(i)}^T] }  \right) V^T \rangle \tilde E_{a(i), b(i)} \right\|_\text{op} \right\} \right],
    \end{align*}
    where $\EE_Z$ denotes taking the expectation with respect to $Z$. At this point, we can apply the definition of the operator norm and proceed with standard symmetrization arguments, resulting in the following chain of inequalities:
    \begin{align*}
        &\text{Replacing operator norm $\| C \|_\text{op}$ with $\sup_{\|z\|_2=1} \langle z, Cz \rangle$:} \\
        &= \EE_{a(i), b(i)} \EE_\sigma \left[\exp\left\{ \xi \sup_{\|z\|_2 = 1} \left\langle z, \sum_{i=1}^m \langle E_{a(i), b(i)}, V\left(u_{\sigma(i)} u_{\sigma(i)}^T - \EE_{\tilde \sigma} [u_{\tilde \sigma(i)} u_{\tilde \sigma(i)}^T] \right)V^T \rangle \tilde E_{a(i), b(i)}  z \right\rangle \right\} \right] \\
        &\text{Pulling out the expectation via $\Phi(\sup_{g \in \cG} \EE|g(X)|) \leq \EE \Phi(\sup_{g \in \cG} |g(X)|)$ for $\Phi = \exp$ convex and non-decreasing:} \\
        &\leq \EE_{a(i), b(i)} \EE_\sigma \boxed{ \EE_{\tilde \sigma} } \left[\exp\left\{ \xi \sup_{\|z\|_2 = 1} \left\langle z, \sum_{i=1}^m \langle E_{a(i), b(i)}, V\left(u_{\sigma(i)} u_{\sigma(i)}^T - u_{\tilde \sigma(i)} u_{\tilde \sigma(i)}^T \right)V^T \rangle \tilde E_{a(i), b(i)} z \right\rangle \right\} \right] \\
        &\text{We can insert Radamacher random variables $\eps$ because $\sigma$ and $\tilde \sigma$ are i.i.d.:} \\
        &= \boxed{ \EE_{\eps} }\EE_{a(i), b(i)} \EE_\sigma \EE_{\tilde \sigma} \left[\exp\left\{ \xi \sup_{\|z\|_2 = 1} \left\langle z, \sum_{i=1}^m \langle E_{a(i), b(i)}, V\boxed{ \eps_i } \left(u_{\sigma(i)} u_{\sigma(i)}^T - u_{\tilde \sigma(i)} u_{\tilde \sigma(i)}^T \right)V^T \rangle \tilde E_{a(i), b(i)} z \right\rangle \right\} \right] \\
        &\text{Splitting the sum via Jensen's inequality:} \\
        &\leq \EE_{\eps, a(i), b(i)} \EE_\sigma \left[\frac{1}{2} \exp\left\{ 2\xi \sup_{\|z\|_2 = 1} \left\langle z, \sum_{i=1}^m \eps_i \langle E_{a(i), b(i)}, V\left(u_{\sigma(i)} u_{\sigma(i)}^T \right)V^T \rangle \tilde E_{a(i), b(i)} z \right\rangle \right\} \right] \\
        &\phantom{=} + \EE_{\eps, a(i), b(i)} \EE_{\tilde \sigma} \left[\frac{1}{2} \exp\left\{ 2\xi \sup_{\|z\|_2 = 1} \left\langle z, \sum_{i=1}^m \eps_i \langle E_{a(i), b(i)}, V\left(u_{\tilde \sigma(i)} u_{\tilde \sigma(i)}^T \right)V^T \rangle \tilde E_{a(i), b(i)} z \right\rangle \right\} \right] \\
        &\text{Removing $\sigma$ and $\tilde \sigma$ by again applying invariance of the expectation to permutation:} \\
        &= \EE_{\eps, a(i), b(i)} \left[\exp\left\{ 2\xi \sup_{\|z\|_2 = 1} \left\langle z, \sum_{i=1}^m \eps_i \langle E_{a(i), b(i)}, V u_i u_i^T V^T \rangle \tilde E_{a(i), b(i)} z \right\rangle \right\} \right]\nonumber \\
        &= \EE \left[\exp\left\{ 2\xi \left\| \sum_{i=1}^m \eps_i X_{i,a(i)} X_{i,b(i)} \tilde E_{a(i), b(i)} \right\|_\text{op} \right\} \right].
    \end{align*}
    We can proceed in exactly the same way for the diagonal terms, resulting in the following inequalities:
    \begin{align*}
        \EE \left[\exp\left\{ \xi \left\| \sum_{i=1}^m (\Theta^*_{a(i), b(i)} - X_{i,a(i)}X_{i,b(i)}) \tilde E_{a(i), b(i)}  \right\|_\text{op} \right\} \right] &\leq \EE \left[\exp\left\{ 2\xi \left\| \sum_{i=1}^m \eps_i X_{i,a(i)} X_{i,b(i)} \tilde E_{a(i), b(i)} \right\|_\text{op} \right\} \right], \\
        \EE \left[\exp\left\{ \xi \left\| \sum_{i=1}^m \frac{1}{2} (\Theta^*_{a(i), a(i)} - X_{i,a(i)}^2) E_{a(i), a(i)} \right\|_\text{op} \right\} \right] &\leq \EE \left[\exp\left\{ 2\xi \left\| \sum_{i=1}^m \frac{1}{2} \eps_i X_{i,a(i)}^2 E_{a(i), a(i)} \right\|_\text{op} \right\} \right].
    \end{align*}
    
    \textbf{Bounding moments}: at this point, we can apply standard matrix Bernstein arguments. First note that for symmetric independent random matrices $Q_i$, we have
    \begin{align*}
        \EE e^{2\xi \| \sum_i Q_i \|_{\text{op}}} &\overset{(i)}{=} \EE \| e^{2\xi \sum_i Q_i} \|_{\text{op}} \overset{(ii)}{\leq} \EE \text{tr}(e^{2\xi \sum_i Q_i}) = \text{tr} (\EE e^{2\xi \sum_i Q_i}) \overset{(iii)}{\leq} \text{tr} (e^{\sum_i \log \EE e^{2\xi Q_i}}),
    \end{align*}
    where (i) follows from the spectral mapping theorem, (ii) from the fact that the matrix exponential $e^Q = \sum_{k=0}^\infty \frac{Q^k}{k!}$ is positive semidefinite, and (iii) from Lemma~{6.13} of \citet{wainwright2019high}. Therefore, it suffices to bound each $\EE e^{2\xi Q_i}$. For ease of notation, we'll define the following random matrices, which are symmetric:
    \begin{align*}
        R_i &= \frac{1}{2} \eps_i X_{i,a(i)} X_{i,b(i)} (E_{a(i), b(i)} + E_{b(i), a(i)}) \\
        S_i &= \eps_i X_{i,a(i)}^2 E_{a(i), a(i)}.
    \end{align*}
    Then, to bound $\EE e^{(2\xi R_i)}$ and $\EE e^{(2\xi S_i)}$, we can bound $\EE R_i^{2n}$ and $\EE S_i^{2n}$ (note that the odd moments are zero because $\eps_i$ is symmetric around the origin). Using our assumption that $\|X\|_\text{max}^2 \leq \alpha$, we can compute these moments as follows:
    \begin{align*}
        \EE R_i^{2n} &= \frac{1}{2^{2n}} X_{i,a(i)}^{2n} X_{i,b(i)}^{2n} \frac{2}{d} I_d \\
        &\preceq \alpha^{2n} \frac{1}{d} I_d \\
        \EE S_i^{2n} &= (X_{i,a(i)}^2)^{2n} \frac{1}{d} I_d \\
        &\preceq \alpha^{2n} \frac{1}{d} I_d,
    \end{align*}
    so $R_i$ and $S_i$ both satisfy the matrix Bernstein condition with $b = \alpha$ and $\text{var}(R_i) \preceq \alpha^2 \frac{1}{d} I_d$. Then, by a matrix Bernstein bound (see, e.g., Lemma~{6.11} of \citet{wainwright2019high}), we have
    \begin{align*}
        \EE e^{2\xi R_i} &\preceq \exp\left\{ \frac{2\xi^2 \text{var}(R_i)}{1 - b|\xi|} \right\}\text{ for all $|\xi| < 1/b$} \\
        &\preceq \exp\left\{ \frac{2\xi^2 \alpha^2 I_d}{d(1 - \alpha|\xi|)} \right\}\text{ for all $|\xi| < 1/\alpha$},
    \end{align*}
    with the same inequality holding for $S_i$. Substituting into the original inequality, for all $|\xi| < 1/\alpha$ we have
    \begin{align*}
        \PP\left( \left\| \frac{1}{m} \sum_{i=1}^m (\Theta^*_{a(i), b(i)} - X_{i, a(i)}X_{i, b(i)}) \tilde E_{a(i), b(i)} \right\|_\text{op} \geq t \right) &\leq \text{tr} (e^{\sum_i \log \EE e^{2\xi R_i}}) e^{-\xi m t} \\
        &\leq \text{tr} \left(\exp\left\{ \frac{2m\xi^2 \alpha^2 I_d}{d(1 - \alpha|\xi|)} \right\}\right) e^{-\xi m t} \\
        &\leq d \exp\left\{ \frac{2m\xi^2 \alpha^2}{d(1 - \alpha|\xi|)} \right\} e^{-\xi m t},
    \end{align*}
    where the last line follows from the fact that $\text{tr}(e^{R}) \leq de^{\|R\|_\text{op}}$ for symmetric matrices $R \in \RR^{d \times d}$. Setting $\xi = \frac{t}{4\alpha^2/d + \alpha t}$ produces the bound
    \begin{align*}
        \PP\left( \left\| \frac{1}{m} \sum_{i=1}^m (\Theta^*_{a(i), b(i)} - X_{i, a(i)}X_{i, b(i)}) \tilde E_{a(i), b(i)} \right\|_\text{op} \leq t \right) \geq 1 - d \text{exp}\left\{ - \frac{m t^2}{8\alpha^2/d + 2\alpha t} \right\},
    \end{align*}
    with the same bound holding for the second and third terms. Finally, for all three bounds to hold simultaneously with probability $\geq 1 - e^{-\delta}$, we can set
    \begin{align*}
        t = 4 \text{max}\left( 2\alpha \sqrt{\frac{\log d + \delta}{dm}}, \alpha \frac{\log d + \delta}{m} \right),
    \end{align*}
    which is dominated by the first term for $m \geq d (\log d + \delta)$.
\end{proof}

\subsection{Restricted strong convexity}\label{subsec:rsc}

In this section, we will lower bound the quantity $\cL(\Theta^* + \Delta) - \cL(\Theta^*) - \langle \nabla \cL(\Theta^*), \Delta \rangle$ with high probability (where the randomness is over the sampled indices). In particular, we'll show that this quantity concentrates around $\frac{1}{d^2} \| \Delta \|_F^2$ through careful analysis of the diagonal and off-diagonal terms, along with peeling arguments similar to those in Theorem~{10.17} of \citet{wainwright2019high} and Theorem~{1} of \citet{negahban2012restricted}. Recall that $\cL$ is defined in Equation~\ref{eqn:loss}; a quick calculation reveals that
\begin{align*}
    \cL(\Theta^* + \Delta) - \cL(\Theta^*) - \langle \nabla \cL(\Theta^*), \Delta \rangle &= \frac{1}{2m}\sum_{i=1}^m [\Delta_{a(i), b(i)}^2 + \Delta_{b(i), a(i)}^2] + [\Delta_{a(i), a(i)}^2 + \Delta_{b(i), b(i)}^2].
\end{align*}
For a matrix $\Delta \in \RR^{d \times d}$, we'll use $P_\text{diag}(\Delta)$ to refer to $\Delta$ with the off-diagonal terms set to zero, and $P_\text{off-diag}(\Delta) = \Delta - P_\text{diag}(\Delta)$ to refer to $\Delta$ with the diagonal set to zero.

\begin{lem}\label{lem:rsc}
Let $(a(1), b(1)), ..., (a(m), b(m))$ be random indices sampled i.i.d.\ uniformly from the set $\{(i,j) : i, j \in [d], i \neq j\}$. Also, let $m \geq d \log d$. Then, for universal constants $c_1$, $c_2$, and $c_3$, we have that for $\cL$ defined in Equation~\ref{eqn:loss}, the following bound holds uniformly for all matrices $\Delta \in \RR^{d \times d}$, with probability $\geq 1 - 2e^{-\delta}$:
\begin{align*}
    \cL(\Theta^* + \Delta) - \cL(\Theta^*) - \langle \nabla \cL(\Theta^*), \Delta \rangle &\geq \frac{1}{d^2} \| \Delta \|_F^2 \\&- c_1 \| \Delta \|_\text{max} \| \Delta \|_\text{nuc} \sqrt{\frac{\log d}{dm}} \\
    &- c_2 \| \Delta \|_\text{max} \| \Delta \|_F \sqrt{\frac{\delta}{dm}} \\
    &- c_3 \| \Delta \|_\text{max}^2 \frac{\delta}{m}.
\end{align*}
\end{lem}
\begin{proof}
    To show this bound, we'll show the following two bounds for the off-diagonal and diagonal terms, which each hold with probability $\geq 1 - e^{-\delta}$ (and therefore together with probability $\geq 1 - 2e^{-\delta}$):
    \begin{enumerate}[(a)]
        \item \textbf{Off-diagonal terms}: with probability $\geq 1 - e^{-\delta}$, the following holds uniformly for all $\Delta \in \RR^{d \times d}$:
        \begin{align*}
            \frac{1}{m}\sum_{i=1}^m \frac{1}{2}[\Delta_{a(i), b(i)}^2 + \Delta_{b(i), a(i)}^2] &\geq \frac{1}{d(d-1)} \| P_\text{off-diag}(\Delta) \|_F^2 \\
            &- c_1 \| P_\text{off-diag}(\Delta) \|_\text{max} \| P_\text{off-diag}(\Delta) \|_\text{nuc} \sqrt{\frac{\log d}{dm}} \\
            &- c_2 \| P_\text{off-diag}(\Delta) \|_\text{max} \| P_\text{off-diag}(\Delta) \|_F \sqrt{\frac{\delta}{d(d-1)m}} \\
            &- c_3 \| P_\text{off-diag}(\Delta) \|_\text{max}^2 \frac{\delta}{m}.
        \end{align*}
        \item \textbf{Diagonal terms}: with probability $\geq 1 - e^{-\delta}$, the following holds uniformly for all $\Delta \in \RR^{d \times d}$:
        \begin{align*}
            \frac{1}{m}\sum_{i=1}^m \frac{1}{2} [\Delta_{a(i)}^2 + \Delta_{b(i)}^2] &\geq \frac{1}{d} \| P_\text{diag}(\Delta) \|_F^2 \\
            &- c_1 \| P_\text{diag}(\Delta) \|_\text{max} \| P_\text{diag}(\Delta) \|_\text{nuc} \sqrt{\frac{\log d}{dm}} \\
            &- c_2 \| P_\text{diag}(\Delta) \|_\text{max} \| P_\text{diag}(\Delta) \|_F \sqrt{\frac{\delta}{dm}} \\
            &- c_3 \| P_\text{diag}(\Delta) \|_\text{max}^2 \frac{\delta}{m}.
        \end{align*}
    \end{enumerate}
    Note that the lemma follows from adding these two claims because we have the following inequalities:
    \begin{enumerate}[(i)]
        \item \textbf{Max}: the max of the entire matrix bounds the max of subsets of the matrix, so $\| P_\text{diag}(\Delta) \|_\text{max} \leq \| \Delta \|_\text{max}$ and $\| P_\text{off-diag}(\Delta) \|_\text{max} \leq \| \Delta \|_\text{max}$.
        \item \textbf{Nuclear norm}: the $P_\text{diag}$ operator reduces nuclear norm, so $\| P_\text{diag}(\Delta) \|_\text{nuc} \leq \| \Delta \|_\text{nuc}$ and $\| P_\text{off-diag}(\Delta) \|_\text{nuc} = \| \Delta - P_\text{diag}(\Delta) \|_\text{nuc} \leq 2\| \Delta \|_\text{nuc}$.
        \item \textbf{Frobenius norm terms}: because $ \| P_\text{off-diag}(\Delta) \|_F^2 + \| P_\text{diag}(\Delta) \|_F^2 = \| \Delta \|_F^2$, we have
        \begin{align*}
            \frac{1}{d(d-1)} \| P_\text{off-diag}(\Delta) \|_F^2 + \frac{1}{d} \| P_\text{diag}(\Delta) \|_F^2 &\geq \frac{1}{d^2} \| P_\text{off-diag}(\Delta) \|_F^2 + \frac{1}{d^2} \| P_\text{diag}(\Delta) \|_F^2 = \frac{1}{d^2} \| \Delta \|_F^2,
        \end{align*}
        lower bounding the first term, and we also have
        \begin{align*}
            \| P_\text{off-diag}(\Delta) \|_F \sqrt{\frac{1}{d(d-1)}} + \| P_\text{diag}(\Delta) \|_F \sqrt{\frac{1}{d}} &\leq \sqrt{ 2\| P_\text{off-diag}(\Delta) \|_F^2 \frac{1}{d(d-1)} + 2\| P_\text{diag}(\Delta) \|_F^2 \frac{1}{d} } \\
            &\leq \sqrt{ 2\| P_\text{off-diag}(\Delta) \|_F^2 \frac{1}{d} + 2\| P_\text{diag}(\Delta) \|_F^2 \frac{1}{d} }\\
            &= \sqrt{\frac{2}{d}} \| \Delta \|_F,
        \end{align*}
        which upper bounds the deviation.
    \end{enumerate}
    
    \textbf{Off-diagonal term bound}: We'll start by bounding the off-diagonal terms; the proof for the diagonal terms will proceed similarly. First, note that the inequality is scale-invariant, so WLOG we can assume $\| P_\text{off-diag}(\Delta) \|_\text{max} = \alpha$. Fixing some $D$ and $\rho$, let $\QQ(D, \rho)$ denote the set
    \begin{align*}
        \QQ(D, \rho) = \{ \Delta \in \RR^{d \times d} :\ &\| P_\text{off-diag}(\Delta) \|_\text{max} = \alpha,\ \| P_\text{off-diag}(\Delta) \|_F \leq D,\ \| P_\text{off-diag}(\Delta) \|_\text{nuc} \leq \rho \}.
    \end{align*}
    and let $Z(D, \rho)$ denote the largest deviation in $\QQ(D, \rho)$, or
    \begin{align*}
        Z(D, \rho) = \sup_{\Delta \in \QQ(D, \rho)} \left| \frac{1}{m}\sum_{i=1}^m \frac{1}{2}[\Delta_{a(i), b(i)}^2 + \Delta_{b(i), a(i)}^2]  - \frac{1}{d(d-1)} \| P_\text{off-diag}(\Delta) \|_F^2 \right|.
    \end{align*}
    We'll first produce a high probability upper bound on $Z(D, \rho)$ for fixed $D$ and $\rho$, and we'll then use a peeling argument to produce a high probability bound for general $D$ and $\rho$.

    \textbf{Bound}: First, note that each summand has expectation
    \begin{align*}
        \EE \left[ \langle \Delta, \tilde E_{a(i), b(i)} \rangle^2 \right] &= \frac{1}{d(d-1)} \| P_\text{off-diag}(\Delta) \|_F^2,
    \end{align*}
    so $Z(D, \rho)$ is an empirical process of the form $\sup_{g \in \cG} | \frac{1}{m} \sum_{i=1}^m g(X_i) - \EE g(X) |$. Furthermore, we have that each term is uniformly bounded $\frac{1}{2}[\Delta_{a(i), b(i)}^2 + \Delta_{b(i), a(i)}^2] \leq \alpha^2$, with uniformly bounded variance:
    \begin{align*}
        \text{var}\left( \langle \Delta, \tilde E_{a(i), b(i)} \rangle^2 \right) &= \EE \langle \Delta, \tilde E_{a(i), b(i)} \rangle^4 - \EE \langle \Delta, \tilde E_{a(i), b(i)} \rangle^2 \\
        &\leq \alpha^2 \EE \langle \Delta, \tilde E_{a(i), b(i)} \rangle^2 - \EE \langle \Delta, \tilde E_{a(i), b(i)} \rangle^2 \\
        &= (\alpha^2 - 1) \frac{1}{d(d-1)} \| P_\text{off-diag}(\Delta) \|_F^2 \\
        &\leq \alpha^2 \frac{1}{d(d-1)} D^2.
    \end{align*}
    Therefore, by a functional Bernstein inequality (Theorem~{3.27} of \citet{wainwright2019high}), we have
    \begin{align*}
        \PP\left( Z(D, \rho) \geq 2\EE Z(D, \rho) + 2 \sigma \sqrt{\frac{\delta'}{m}} + 2 b \frac{\delta'}{m}\right) &\leq e^{-\delta'}
    \end{align*}
    where $\sigma = \alpha D \sqrt{\frac{1}{d(d-1)}}$ and $b = \alpha^2$. Next, to bound the expectation $\EE Z(D, \rho)$, we have that
    \begin{align*}
        \EE \sup_{\Delta \in \QQ(D, \rho)} \left| \frac{1}{m}\sum_{i=1}^m \langle \tilde E_{a(i), b(i)}, \Delta \rangle^2 - \EE[ \langle \tilde E_{a(i), b(i)}, \Delta \rangle^2 ]  \right| &\overset{(i)}{\leq} 2 \EE \sup_{\Delta \in \QQ(D, \rho)} \left| \frac{1}{m}\sum_{i=1}^m \eps_i \langle \tilde E_{a(i), b(i)}, \Delta \rangle^2 \right| \\
        &\overset{(ii)}{=} 2 \EE \sup_{\Delta \in \QQ(D, \rho)} \left| \frac{1}{m}\sum_{i=1}^m \eps_i \langle \tilde E_{a(i), b(i)}, P_\text{off-diag}(\Delta) \rangle^2 \right| \\
        &\overset{(iii)}{\leq} 4 \alpha \EE \sup_{\Delta \in \QQ(D, \rho)} \left| \frac{1}{m}\sum_{i=1}^m \eps_i \langle \tilde E_{a(i), b(i)}, P_\text{off-diag}(\Delta) \rangle \right| \\
        &\overset{(iv)}{\leq} 4 \alpha \EE \sup_{\Delta \in \QQ(D, \rho)} \left| \left\| \frac{1}{m}\sum_{i=1}^m \eps_i \tilde E_{a(i), b(i)} \right\|_\text{op} \| P_\text{off-diag}(\Delta) \|_\text{nuc} \right| \\
        &\leq 4 \alpha \rho \EE \left\| \frac{1}{m}\sum_{i=1}^m \eps_i \tilde E_{a(i), b(i)} \right\|_\text{op} \\
        &\overset{(v)}{\leq} 64\alpha \rho \sqrt{\frac{\log d}{dm}},
    \end{align*}
    where (i) follows from Radamacher symmetrization, (ii) follows from the fact that $a(i) \neq b(i)$ so only off-diagonal terms of $\Delta$ are sampled, (iii) follows from the fact that $\| P_\text{off-diag}(\Delta) \|_\text{max} \leq \alpha$ and the Ledoux-Talagrand contraction inequality for Radamacher processes (see (5.61) in \citet{wainwright2019high} or Section~{4.2} in \citet{ledoux1991probability}), and (iv) follows from Holder's inequality. To show (v), note that each $\eps_i \tilde E_{a(i), b(i)}$ is mean zero with operator norm $1$ and variance $\text{var} ( \eps_i \tilde E_{a(i), b(i)} ) = \frac{1}{d} I_d$, so by matrix Bernstein (Theorem~{6.17} of \citet{wainwright2019high}) we have
    \begin{align*}
        \PP \left( \left\| \frac{1}{m}\sum_{i=1}^m \eps_i \tilde E_{a(i), b(i)} \right\|_\text{op} \geq t \right) \leq 2d\exp\left\{ - \frac{mt^2}{2(\frac{1}{d} + t)} \right\}.
    \end{align*}
    Then, for $m \geq d \log d$, we can integrate to bound the expectation by $16 \sqrt{\frac{\log d}{dm}}$: in particular, by Exercise~{2.8(a)} of \citet{wainwright2019high}, we have that
    \begin{align*}
        \PP(Z \geq t) \leq C e^{-\frac{t^2}{2(\nu^2 + bt)}} \implies \EE Z &\leq 2 \nu (\sqrt{\pi} + \sqrt{\log C}) + 4b (1 + \log C)
    \end{align*}
    where we have $C = 2d$, $\nu^2 = \frac{1}{dm}$, and $b = \frac{1}{m}$, resulting in
    \begin{align*}
        \EE \left\| \frac{1}{m}\sum_{i=1}^m \eps_i \tilde E_{a(i), b(i)} \right\|_\text{op} &\leq 2 \frac{1}{\sqrt{dm}}(\sqrt{\pi} + \sqrt{\log(2d)}) + 4 \frac{1}{m}(1 + \log (2d)) \\
        &\leq 16 \sqrt{\frac{\log d}{dm}},
    \end{align*}
    where the second inequality follows from combining terms and using the fact that the first term is larger when $m \geq d \log d$. Putting everything together, we have that
    \begin{align}\label{eqn:rsc-part-1}
        \PP\left( Z(D, \rho) \leq 128 \alpha \rho \sqrt{\frac{\log d}{dm}} + 4 \alpha D \sqrt{\frac{\delta'}{d(d-1)m}} + 4 \alpha^2\frac{\delta'}{m} \right) \geq 1 - e^{-\delta'}.
    \end{align}

    \textbf{Peeling}: given this bound, we can use a peeling argument to extend to general $D$ and $\rho$. Our approach will be to cover all possible $\Delta$ with the sets
    \begin{align*}
        \QQ_{k,\ell} = \{ \Delta \in \RR^{d \times d} :\ &\| P_\text{off-diag}(\Delta) \|_\text{max} = \alpha,\ \alpha 2^{k - 1} \leq \| P_\text{off-diag}(\Delta) \|_F \leq \alpha 2^{k},\ \alpha 2^{\ell - 1} \leq \| P_\text{off-diag}(\Delta) \|_\text{nuc} \leq \alpha 2^{\ell} \}.
    \end{align*}
    The idea is that we can use the union bound to ensure that the uniform bound in the first part (Equation~\ref{eqn:rsc-part-1}) applies to each set $\QQ_{k,\ell}$. Then, because $\QQ_{k,\ell}$ captures values of $D$ and $\rho$ up to factors of 2, we can produce the desired inequality for general $\Delta$ while only losing constant factors in the deviation terms.
    
    The first step is to bound the number of such sets that we need, which we can do by upper- and lower-bounding the Frobenius and nuclear norms with respect to $\alpha$:
    \begin{align*}
        \alpha &= \| P_\text{off-diag} (\Delta) \|_\text{max} \leq \| P_\text{off-diag} (\Delta) \|_F \leq d \| P_\text{off-diag} (\Delta) \|_\text{max} = d \alpha \\
        \alpha &= \| P_\text{off-diag} (\Delta) \|_\text{max} \leq \| P_\text{off-diag} (\Delta) \|_\text{nuc} \leq d^{3/2} \| P_\text{off-diag} (\Delta) \|_\text{max} = d^{3/2} \alpha.
    \end{align*}
    Therefore, it suffices to have $k = 1, 2, ..., \lceil \log d \rceil$ and $\ell = 1, 2, ..., \lceil (3/2) \log d \rceil$. By the union bound, the probability that the bound in Equation~\ref{eqn:rsc-part-1} holds for all of the sets $\QQ_{k,\ell}$ is at least $\geq 1 - \lceil (3/2) \log d \rceil \lceil \log d \rceil 2\exp\left\{- \delta' \right\}$, which we can bound by $\geq 1 - e^{-\delta}$ for $\delta \geq \log 6 + \log \log d$ by setting $\delta' = 3 \delta$. 
    
    Then, for any specific $\Delta$, letting $k, \ell$ be the indices such that $\Delta \in \QQ_{k, \ell}$, we have that
    \begin{align*}
        \frac{1}{m}\sum_{i=1}^m \langle \Delta, \tilde E_{a(i), b(i)} \rangle^2  &\overset{(i)}{\geq} \frac{1}{d(d-1)}\| P_\text{off-diag}(\Delta) \|_F^2 - c_1 \alpha \rho \sqrt{\frac{\log d}{dm}} - c_2 \alpha D \sqrt{\frac{\delta}{d(d-1)m}} - c_3 \alpha^2\frac{\delta}{m} \\
        &= \frac{1}{d(d-1)}\| P_\text{off-diag}(\Delta) \|_F^2 - c_1 \alpha (\alpha 2^\ell) \sqrt{\frac{\log d}{dm}} - c_2 \alpha (\alpha 2^k) \sqrt{\frac{\delta}{d(d-1)m}} - c_3 \alpha^2\frac{\delta}{m} \\
        & \overset{(ii)}{\geq} \frac{1}{d(d-1)}\| P_\text{off-diag}(\Delta) \|_F^2 - 2c_1 \alpha \| P_\text{off-diag}(\Delta) \|_\text{nuc} \sqrt{\frac{\log d}{dm}} \\
        &- 2c_2 \alpha \| P_\text{off-diag}(\Delta) \|_F \sqrt{\frac{\delta}{d(d-1)m}} - c_3 \alpha^2\frac{\delta}{m},
    \end{align*}
    where (i) follows from the fact that $\QQ_{k, \ell} \subseteq \QQ(D, \rho)$ for $D = \alpha2^k, \rho = \alpha2^\ell$, and (ii) follows from the fact that $\alpha 2^{k-1} \leq \| P_\text{off-diag}(\Delta) \|_F$ and $\alpha 2^{\ell-1} \leq \| P_\text{off-diag}(\Delta) \|_\text{nuc}$.
    
    \textbf{Diagonal terms}: the bound for the diagonal terms proceeds in the same way, but with slightly different quantities. As before, we'll assume that $\|P_\text{diag}(\Delta)\|_\text{max} = \alpha$. Fixing $D$ and $\rho$, we'll define the set $\QQ(D, \rho)$ as
    \begin{align*}
        \QQ(D, \rho) = \{ \Delta \in \RR^{d \times d} :\ &\| P_\text{diag}(\Delta) \|_\text{max} = \alpha,\ \| P_\text{diag}(\Delta) \|_F \leq D,\ \| P_\text{diag}(\Delta) \|_\text{nuc} \leq \rho \},
    \end{align*}
    and we'll define $Z(D, \rho)$ as
    \begin{align*}
        Z(D, \rho) = \sup_{\Delta \in \QQ(D, \rho)} \left| \frac{1}{m}\sum_{i=1}^m \frac{1}{2}[\Delta_{a(i), a(i)}^2 + \Delta_{b(i), b(i)}^2]  - \frac{1}{d} \| P_\text{diag}(\Delta) \|_F^2 \right|,
    \end{align*}
    where we have that each term has expectation
    \begin{align*}
        \EE \left[ \frac{1}{2}[\Delta_{a(i), a(i)}^2 + \Delta_{b(i), b(i)}^2] \right] = \frac{1}{d} \| P_\text{diag}(\Delta) \|_F^2.
    \end{align*}
    We again have that each term is uniformly bounded by $\alpha^2$ and has uniformly bounded variance, but this time with $\alpha^2 \frac{1}{d} D^2$. The bound for $\EE Z(D, \rho) \leq 64 \alpha \rho \sqrt{\frac{\log d}{dm}}$ proceeds exactly as before, producing the functional Bernstein inequality
    \begin{align*}
        \PP\left( Z(D, \rho) \leq 128 \alpha \rho \sqrt{\frac{\log d}{dm}} + 4 \alpha D \sqrt{\frac{\delta'}{dm}} + 4 \alpha^2\frac{\delta'}{m} \right) \geq 1 - e^{-\delta'}.
    \end{align*}
    Finally, we can proceed with exactly the same peeling argument as before, leading to the following bound holding uniformly with probability $\geq 1 - e^{-\delta}$:
    \begin{align*}
        \frac{1}{m}\sum_{i=1}^m \frac{1}{2}[\Delta_{a(i), a(i)}^2 + \Delta_{b(i), b(i)}^2] &\geq \frac{1}{d} \| P_\text{diag}(\Delta) \|_F^2 - 2c_1 \alpha \| P_\text{diag}(\Delta) \|_\text{nuc} \sqrt{\frac{\log d}{dm}} \\
        &- 2c_2 \alpha \| P_\text{diag}(\Delta) \|_F \sqrt{\frac{\delta}{dm}} - c_3 \alpha^2\frac{\delta}{m},
    \end{align*}
    which completes the proof.
\end{proof}

\subsection{Decomposability}\label{subsec:decomp}
In this section, we show that for large enough regularization strength $\lambda$, the error $\Delta = \hat \Theta - \Theta^*$ must have small nuclear norm. The arguments in this section proceed exactly as in Proposition 9.13 of \citet{wainwright2019high}, which we reproduce here for completeness. To aid our analysis of the nuclear norm regularizer, we'll first define the subspaces $\MM$ and $\bar \MM$ of $\RR^{d \times d}$ as follows (where $\mathbb{S}^\perp$ denotes the orthogonal complement of a subspace $\mathbb{S}$), following~\citet{negahban2012unified}:
\begin{align*}
    \MM &= \{ \Theta :\ \text{rowspace}(\Theta) \subseteq \text{rowspace}(\Theta^*),\ \text{colspace}(\Theta) \subseteq \text{colspace}(\Theta^*) \} \\
    \bar \MM^\perp &= \{ \Theta :\ \text{rowspace}(\Theta) \subseteq \text{rowspace}(\Theta^*)^\perp,\ \text{colspace}(\Theta) \subseteq \text{colspace}(\Theta^*)^\perp \}.
\end{align*}
From these subspaces, we can define $\MM^\perp$ and $\bar \MM$ accordingly as their orthogonal complements. As an example, if $\Theta^*$ is the rank-$r$ matrix with the $r \times r$ identity in the top left corner and zeros otherwise, then we have
\begin{align*}
    \MM &= \begin{bmatrix*}[l]
        \Gamma_{r \times r} & 0 \\
        0 & 0
    \end{bmatrix*} \\
    \MM^\perp &= \begin{bmatrix*}[l]
        0 & \Gamma_{r \times (d-r)} \\
        \Gamma_{(d-r) \times r} & \Gamma_{(d-r) \times (d-r)}
    \end{bmatrix*} \\
    \bar \MM^\perp &= \begin{bmatrix*}[l]
        0 & 0 \\
        0 & \Gamma_{(d-r) \times (d-r)}
    \end{bmatrix*} \\
    \bar \MM &= \begin{bmatrix*}[l]
        \Gamma_{r \times r} & \Gamma_{r \times (d-r)} \\
        \Gamma_{(d-r) \times r} & 0
    \end{bmatrix*},
\end{align*}
where each $\Gamma_{a \times b}$ represents an arbitrary matrix in $\RR^{a \times b}$. With these subspaces defined, we have the following two facts: (1) for any $A \in \MM$ and $B \in \bar \MM^\perp$, we have $\| A + B \|_\text{nuc} = \| A \|_\text{nuc} + \| B \|_\text{nuc}$, i.e.\ the nuclear norm is decomposable with respect to $(\MM, \bar \MM)$~\citep{negahban2012unified}, and (2) if $\Theta^*$ is rank $r$, then all matrices in $\MM$ are at most rank $r$ and all matrices in $\bar \MM$ are at most rank $2r$. Broadly speaking, the proof will proceed by using the optimality of $\hat \Theta$ (with large enough regularization strength $\lambda$) to bound $\| \Delta \|_\text{nuc}$, which involves projecting $\Delta$ onto these subspaces and using decomposability of the nuclear norm.
\begin{lem}\label{lem:decomp}[Proposition 9.13 of \citet{wainwright2019high}]
Let $\Delta = \hat \Theta - \Theta^*$ denote the error, where $\hat \Theta$ solves the optimization problem defined in Equation~\ref{eqn:program}. Also, suppose that the regularization strength $\lambda$ is at least
\begin{align*}
    \lambda \geq 2 \| \nabla \cL(\Theta^*) \|_\text{op}.
\end{align*}
Then, we have
\begin{align*}
    \| \Delta \|_\text{nuc} \leq 4 \sqrt{2r} \| \Delta \|_F.
\end{align*}
\end{lem}
\begin{proof}
    First, note that we can project the error $\Delta$ onto orthogonal subspaces $\bar \MM$ and $\bar \MM^\perp$ as follows:
    \begin{align*}
        \| \Delta \|_\text{nuc} &= \| \Delta_{\bar \MM^\perp} + \Delta_{\bar \MM} \|_\text{nuc} \\
        \text{(triangle)}\ &\leq \| \Delta_{\bar \MM^\perp} \|_\text{nuc} + \| \Delta_{\bar \MM} \|_\text{nuc} \\
        &\overset{(i)}{\leq} \| \Delta_{\bar \MM^\perp} \|_\text{nuc} + \sqrt{2r} \| \Delta_{\bar \MM} \|_F \\
        &\leq \| \Delta_{\bar \MM^\perp} \|_\text{nuc} + \sqrt{2r} \| \Delta \|_F,
    \end{align*}
    where (i) follows from the fact that any matrix in $\bar \MM$ is at most rank $2r$. Therefore, at this point it suffices to bound $\| \Delta_{\bar \MM^\perp} \|_\text{nuc}$. By the optimality of $\hat \Theta$, we have
    \begin{align*}
        0 &\geq L(\Theta^* + \Delta) - L(\Theta^*) + \lambda \left( \| \Theta^* + \Delta \|_\text{nuc} - \| \Theta^* \|_\text{nuc} \right) \\
        &\overset{(i)}{\geq} \langle \nabla L(\Theta^*), \Delta \rangle + \lambda \left( \| \Theta^* + \Delta \|_\text{nuc} - \| \Theta^* \|_\text{nuc} \right) \\
        &\overset{(ii)}{\geq} -\| \nabla L(\Theta^*) \|_\text{op} \| \Delta \|_\text{nuc} + \lambda \left( \| \Theta^* + \Delta \|_\text{nuc} - \| \Theta^* \|_\text{nuc} \right) \\
        &\overset{(iii)}{\geq} - \frac{\lambda}{2} \| \Delta \|_\text{nuc} + \lambda \left( \| \Theta^* + \Delta \|_\text{nuc} - \| \Theta^* \|_\text{nuc} \right),
    \end{align*}
    where (i) follows from the convexity of $\cL$, (ii) from Holder's inequality, and (iii) from our assumption that $\lambda \geq 2 \| \nabla \cL (\Theta^*) \|_\text{op}$. Next, we'll project the error $\Delta$ onto $\bar \MM$ and $\bar \MM^\perp$ to expand the second term as follows:
    \begin{align*}
        \| \Theta^* + \Delta \|_\text{nuc} - \| \Theta^* \|_\text{nuc} &= \| \Theta^* + \Delta_{\bar \MM^\perp} + \Delta_{\bar \MM} \|_\text{nuc} - \| \Theta^* \|_\text{nuc} \\
        &\overset{(i)}{\geq} \| \Theta^* + \Delta_{\bar \MM^\perp} \|_\text{nuc} - \| \Delta_{\bar \MM} \|_\text{nuc} - \| \Theta^* \|_\text{nuc} \\
        &\overset{(ii)}{=} \| \Theta^* \|_\text{nuc} + \| \Delta_{\bar \MM^\perp} \|_\text{nuc} - \| \Delta_{\bar \MM} \|_\text{nuc} - \| \Theta^* \|_\text{nuc} \\
        &= \| \Delta_{\bar \MM^\perp} \|_\text{nuc} - \| \Delta_{\bar \MM} \|_\text{nuc},
    \end{align*}
    where (i) follows from reverse triangle, and (ii) follows from decomposability and the fact that $\Theta^* \in \MM$. Substituting the expression for the second term, we have that
    \begin{align*}
        0 &\geq -\frac{\lambda}{2} \| \Delta \|_\text{nuc} + \lambda (\| \Delta_{\bar \MM^\perp} \|_\text{nuc} - \| \Delta_{\bar \MM} \|_\text{nuc})\\
        &\overset{(i)}{\geq} -\frac{\lambda}{2} (\| \Delta_{\bar \MM^\perp} \|_\text{nuc} + \| \Delta_{\bar \MM} \|_\text{nuc}) + \lambda (\| \Delta_{\bar \MM^\perp} \|_\text{nuc} - \| \Delta_{\bar \MM} \|_\text{nuc}) \\
        &= \frac{\lambda}{2} (\| \Delta_{\bar \MM^\perp} \|_\text{nuc} - 3\| \Delta_{\bar \MM} \|_\text{nuc}),
    \end{align*}
    where (i) follows from the triangle inequality. Rearranging, we have that $\|\Delta_{\bar \MM^\perp}\|_\text{nuc} \leq 3 \| \Delta_{\bar \MM} \|_\text{nuc}$. Putting everything together, we have that $\| \Delta \|_\text{nuc} \leq 4 \| \Delta_{\bar \MM} \|_\text{nuc} \leq 4 \sqrt{2r} \| \Delta \|_F$ as desired.
\end{proof}

\subsection{Proof of theorem}\label{subsec:thmproof}
Finally, we can put everything together to prove Theorem~\ref{thm:main}, which we reproduce here:
\\\\
\noindent \textbf{Theorem}. {\itshape Let $\hat \Theta$ be the solution of the optimization problem defined in Equation~\ref{eqn:program}, where $\lambda$ is set to $16 \alpha \sqrt{\frac{ \log d + \delta}{dm}}$. Also, suppose that $X$ is rank $r$ with $\| X \|_\text{max}^2 \leq \alpha$, and $m \geq d(\log d + \delta)$. Then, with probability $\geq 1 - 3e^{-\delta}$, we have that
\begin{align*}
    \frac{1}{d^2} \| \hat \Theta - \Theta^* \|_F^2 \lesssim \alpha^2 \frac{rd(\log d + \delta)}{m}.
\end{align*}
}

\begin{proof}
By our setting of $\lambda = 16 \alpha \sqrt{\frac{\log d + \delta}{dm}}$, we can apply Lemma~\ref{lem:opnormbound} (operator norm bound) to show that
\begin{align*}
    \frac{\lambda}{2} \geq \| \nabla \cL(\Theta^*) \|_\text{op}
\end{align*}
with probability $\geq 1 - e^{-\delta}$. We'll also condition on the high-probability bound in Lemma~\ref{lem:rsc} (restricted strong convexity), which holds with probability $\geq 1 - 2e^{-\delta}$ (so by the union bound, both hold with probability $\geq 1 - 3 e^{-\delta}$):
\begin{align*}
    \cL(\Theta^* + \Delta) - \cL(\Theta^*) - \langle \nabla \cL(\Theta^*), \Delta \rangle &\geq \frac{1}{d^2} \| \Delta \|_F^2 - c_1 \| \Delta \|_\text{max} \| \Delta \|_\text{nuc} \sqrt{\frac{\log d}{dm}} \\
    &- c_2 \| \Delta \|_\text{max} \| \Delta \|_F \sqrt{\frac{\delta}{dm}} - c_3 \| \Delta \|_\text{max}^2 \frac{\delta}{m}.
\end{align*}
First, note that we have $\| \Delta \|_\text{max} \leq \| \hat \Theta \|_\text{max} + \| \Theta^* \|_\text{max} \leq 2\alpha$, where $\| \hat \Theta \|_\text{max} \leq \alpha$ is a constraint of the optimization problem (Equation~\ref{eqn:program}) and $\| \Theta^* \|_\text{max} \leq \alpha$ follows from our assumption that $\| X \|_\text{max}^2 \leq \alpha$:
\begin{align*}
    \max_{ij} | \Theta^*_{ij} | &= \max_{ij} \left| \frac{1}{m} \sum_{k=1}^m X_{ki} X_{kj} \right| \\
    &\leq \max_{ij} \frac{1}{m} \sum_{k=1}^m | X_{ki} | | X_{kj} | \\
    &\leq \frac{1}{m} \sum_{k=1}^m \alpha.
\end{align*}
Then, given these lemmas, we can prove the result as follows: first, by the optimality of $\hat \Theta$, we have
\begin{align*}
    0 &\geq \cL(\Theta^* + \Delta) - \cL(\Theta^*) + \lambda ( \| \Theta^* + \Delta \|_\text{nuc} - \| \Theta^* \|_\text{nuc} ) \\
    &\overset{(i)}{\geq} \boxed{ \langle \nabla \cL(\Theta^*), \Delta \rangle + \frac{1}{d^2} \| \Delta \|_F^2 - c_1 \alpha \| \Delta \|_\text{nuc} \sqrt{\frac{\log d}{dm}} - c_2 \alpha \| \Delta \|_F \sqrt{\frac{\delta}{dm}} - c_3 \alpha^2 \frac{\delta}{m} }\\
    &+ \lambda ( \| \Theta^* + \Delta \|_\text{nuc} - \| \Theta^* \|_\text{nuc} ) \\
    &\overset{(ii)}{\geq} \boxed{ -\| \nabla \cL(\Theta^*) \|_\text{op} \| \Delta \|_\text{nuc} } + \frac{1}{d^2} \| \Delta \|_F^2 - c_1 \alpha \| \Delta \|_\text{nuc} \sqrt{\frac{\log d}{dm}} - c_2 \alpha \| \Delta \|_F \sqrt{\frac{\delta}{dm}} - c_3 \alpha^2 \frac{\delta}{m} \\
    &+ \lambda ( \| \Theta^* + \Delta \|_\text{nuc} - \| \Theta^* \|_\text{nuc} ) \\
    &\overset{(iii)}{\geq} \boxed{ -\frac{\lambda}{2} \| \Delta \|_\text{nuc} } + \frac{1}{d^2} \| \Delta \|_F^2 - c_1 \alpha \| \Delta \|_\text{nuc} \sqrt{\frac{\log d}{dm}} - c_2 \alpha \| \Delta \|_F \sqrt{\frac{\delta}{dm}} - c_3 \alpha^2 \frac{\delta}{m} \\
    &+ \lambda ( \| \Theta^* + \Delta \|_\text{nuc} - \| \Theta^* \|_\text{nuc} ),
\end{align*}
where (i) follows from Lemma~\ref{lem:rsc} (restricted strong convexity), (ii) from Holder's inequality, and (iii) from our setting of $\lambda$ and Lemma~\ref{lem:opnormbound} (operator norm bound). Next, we can combine terms involving $\lambda$ and use our bound on the nuclear norm (Lemma~\ref{lem:decomp}), producing
\begin{align*}
    &\overset{(i)}{\geq} -\frac{\lambda}{2} \| \Delta \|_\text{nuc} + \frac{1}{d^2} \| \Delta \|_F^2 - c_1 \alpha \| \Delta \|_\text{nuc} \sqrt{\frac{\log d}{dm}} - c_2 \alpha \| \Delta \|_F \sqrt{\frac{\delta}{dm}} - c_3 \alpha^2 \frac{\delta}{m} + \boxed{ \lambda (-\|\Delta\|_\text{nuc}) } \\
    &= \frac{1}{d^2} \| \Delta \|_F^2 - c_1 \alpha \| \Delta \|_\text{nuc} \sqrt{\frac{\log d}{dm}} - c_2 \alpha \| \Delta \|_F \sqrt{\frac{\delta}{dm}} - c_3 \alpha^2 \frac{\delta}{m} - \boxed{ \frac{3}{2} \lambda \|\Delta\|_\text{nuc} } \\
    &\overset{(ii)}{\geq} \frac{1}{d^2} \| \Delta \|_F^2 - c_1 \alpha \| \Delta \|_\text{nuc} \sqrt{\frac{\log d}{dm}} - c_2 \alpha \| \Delta \|_F \sqrt{\frac{\delta}{dm}} - c_3 \alpha^2 \frac{\delta}{m} - c_4 \|\Delta\|_\text{nuc} \boxed{ \alpha \sqrt{\frac{\log d + \delta}{dm}} } \\
    &\overset{(iii)}{\geq} \frac{1}{d^2} \| \Delta \|_F^2 - c_1 \alpha\ \boxed{ \sqrt{r} \| \Delta \|_F } \sqrt{\frac{\log d}{dm}} - c_2 \alpha \| \Delta \|_F \sqrt{\frac{\delta}{dm}} - c_3 \alpha^2 \frac{\delta}{m} - c_4\ \boxed{ \sqrt{r} \| \Delta \|_F }\ \alpha \sqrt{\frac{\log d + \delta}{dm}},
\end{align*}
where (i) follows from reverse triangle, (ii) from our setting of $\lambda$, and (iii) from Lemma~\ref{lem:decomp} (nuclear norm bound). At this point, we can rearrange to produce the bound
\begin{align*}
    \frac{1}{d^2} \| \Delta \|_F^2 \lesssim \text{max} \left( \alpha \sqrt{r} \| \Delta \|_F \sqrt{\frac{\log d}{dm}},\ \alpha \| \Delta \|_F \sqrt{\frac{\delta}{dm}},\ \alpha^2 \frac{\delta}{m},\ \sqrt{r} \| \Delta \|_F \alpha \sqrt{\frac{\log d + \delta}{dm}} \right),
\end{align*}
which we can simplify into
\begin{align*}
    \frac{1}{d} \| \Delta \|_F &\lesssim \text{max} \left( \alpha \sqrt{\frac{rd\log d}{m}},\ \alpha \sqrt{\frac{d\delta}{m}},\ \alpha \sqrt{\frac{\delta}{m}},\ \alpha \sqrt{\frac{rd(\log d + \delta)}{m}} \right) \\
    &\lesssim \alpha \sqrt{\frac{rd(\log d + \delta)}{m}},
\end{align*}
completing the proof.
\end{proof}

\section{Experiments and hyperparameters}\label{subsec:hyperparams}
Experiments were run on TITAN RTX and RTX 3090 GPUs with 24 gigabytes of memory; in all experiments we set the random seed to zero. Optimization was done via Adam~\citep{kingma:adam} with $\text{lr} = 1\mathrm{e}{-10},\ \beta = (0.9, 0.999)$, and $10{,}000$ steps.

\end{document}